%% file: iclr2024_conference.tex
\documentclass{article} 
\usepackage{iclr2024_conference,times}

\input{math_commands.tex}

\usepackage{hyperref}
\usepackage{url}

\usepackage{enumitem}
\usepackage{wrapfig}
\usepackage{algorithm}
\usepackage{algpseudocode}
\usepackage{amssymb, amsmath, amsthm}
\usepackage{graphics}
\usepackage{epsfig}
\usepackage{graphicx}
\usepackage{subfigure}
\usepackage{booktabs}
\usepackage{float}
\usepackage{color}
\usepackage{enumitem}
\usepackage{xurl}
\usepackage{multirow}
\newcommand{\real}{\mathbb{R}}
\renewcommand{\paragraph}[1]{\textbf{#1.}}
\newcommand{\ourmodel}{RGCCL}
\usepackage{stfloats}

\newtheorem{assumption}{Assumption}[section]
\newtheorem{definition}{Definition}
\newtheorem{theorem}{Theorem}[section]
\newtheorem{lemma}[theorem]{Lemma}

\newtheorem{proposition}{Proposition}

\title{Understanding Community Bias Amplification in Graph Representation Learning}

\iclrfinalcopy
\author{%
	\textbf{Shengzhong Zhang$^{\text{1}}$}~~~
	\textbf{Wenjie Yang$^{\text{1}}$}~~~
	\textbf{Yimin Zhang$^{\text{2}}$}~~~
	\textbf{Hongwei Zhang$^{\text{1}}$}\\
	\textbf{Divin Yan$^{\text{1}}$}~~~
	\textbf{Zengfeng Huang$^{\text{1}}$\thanks{Corresponding Author. Email: huangzf@fudan.edu.cn}}\\
	$^{\text{1}}$Fudan University~~
	$^{\text{2}}$Ant Group
}

\begin{document}

\maketitle

\begin{abstract}
In this work, we discover a phenomenon of community bias amplification in graph representation learning, which refers to the exacerbation of performance bias between different classes by graph representation learning. We conduct an in-depth theoretical study of this phenomenon from a novel spectral perspective. 
Our analysis suggests that structural bias between communities results in varying local convergence speeds for node embeddings. This phenomenon leads to bias amplification in the classification results of downstream tasks. Based on the theoretical insights, we propose random graph coarsening, which is proved to be effective in dealing with the above issue. Finally, we propose a novel graph contrastive learning model called Random Graph Coarsening Contrastive Learning (\ourmodel), which utilizes random coarsening as data augmentation and mitigates community bias by contrasting the coarsened graph with the original graph. Extensive experiments on various datasets demonstrate the advantage of our method when dealing with community bias amplification. 
\end{abstract}

\input{intro}

\input{preliminaries}

\input{convergence2}

\input{method}

\input{divergence3}
\input{related}

\input{exp}

\section{Conclusion}
In this paper, we study the community bias amplification in unsupervised graph representation learning. We presents a novel perspective on this problem through the lens of convergence bias and embedding density imbalance, and a comprehensive theoretical analysis is provided. Based on our theoretical insights, we propose to use random graph coarsening to mitigate this issue, and give theoretical guidance on how to design effective random coarsening algorithms. Finally, a graph contrastive learning model is proposed which utilize random graph coarsening as graph augmentation and a loss function is designed for this new form of graph augmentation.

\bibliography{paper}
\bibliographystyle{iclr2024_conference}

\newpage
\appendix

\input{appendix.tex}

\end{document}

%% file: math_commands.tex
\usepackage{amsmath,amsfonts,bm}

\def\eqref#1{equation~\ref{#1}}

\def\1{\bm{1}}

\DeclareMathAlphabet{\mathsfit}{\encodingdefault}{\sfdefault}{m}{sl}
\SetMathAlphabet{\mathsfit}{bold}{\encodingdefault}{\sfdefault}{bx}{n}



%% file: intro.tex
\section{Introduction}

Graph representation learning (GRL) aims to generate embedding vectors capturing both the structure and feature information. Graph neural networks (GNNs) are the primary encoder architecture for deep GRL \citep{bojchevski2018deep, zhu2020grace,zhu2021gca,zhang2021ccassg,zheng2022ggd}, which are often trained with unsupervised graph contrastive objectives. Such methods are called graph contrastive learning (GCL) and exhibit outstanding performance in various downstream tasks. Compared to other unsupervised methods, the distinctiveness of GCL lies in the encoder's simultaneous use of structural and feature information. Due to the encoder's use of structural information, the final embeddings are likely to inherit the structural bias in the graph, which may cause undesirable performance unfairness in downstream tasks. This phenomenon is demonstrated in Figure \ref{fig:exp}, where we compare the node classification performance of MLP (utilizing feature information only) and the state-of-the-art GCL model GGD \citep{zheng2022ggd}. Although GGD has a much better overall accuracy than MLP, GGD exhibits greater performance differences between different classes of nodes. In other words, GCL exacerbates the performance bias between different classes.

We refer to the phenomenon exhibited in Figure \ref{fig:exp} as \emph{community bias amplification}. This exacerbated bias arises from local structural disparities among classes of nodes and is unrelated to labels and other information. Graph structural bias problems have been studied in previous works \citep{feng2020invest, kang2022rawlsgcn, liu2023on, wang2022grade}. However, they focused on the structure of individual nodes such as degrees and the distance to class boundaries. On the other hand, we study community bias, which is a collective structural bias issue of the entire community. Community bias studied in this work should also be distinguished from the class imbalance problem \citep{song2022tam}. Although both consider the collective bias of communities, works on class-imbalanced classification focus on (semi-)supervised learning and aim to reduce prediction bias caused by imbalanced label distributions.

This paper investigates the following two questions: \emph{1. Why community bias amplification exists in existing GCL methods? 2. Can we design new GCL models to alleviate this issue?} To answer the first question, we analyze the structural bias problem from a spectral perspective, which provides a theoretical explanation on the causes of  community bias amplification in existing GCL models. There have been numerous works conducting spectral analysis of GNNs, e.g., \citep{kipf2016semi,wu2019simplifying,oono2020asymptotic,rong2019dropedge}.However, existing analyses mainly focus on the global behavior: They try to characterize the distribution of node representations using the spectrum of the message passing operator. We point out that this is not suitable when structural bias exists across different regions of the graph. If the number of layers in a GNN is not too large as prevalent in real applications, the embedding distributions of different communities are better characterized by their local spectrum. In particular, if the structures of two communities differ a lot, then the second largest eigenvalues of their normalized adjacency matrices can be quite different, which leads to very different convergence speeds to the stationary subspace. As a result, the embedding distributions of the two communities exhibit different levels of concentration. We then show that such an imbalance in the embedding densities can cause unfairness in downstream tasks through a natural statistical model.

\begin{figure*}[t]
	\setlength{\abovecaptionskip}{-0.1cm}
	\setlength{\belowcaptionskip}{-0.5cm}
	\centering
  \subfigure[MLP on Cora]{\includegraphics[width=0.24\textwidth]{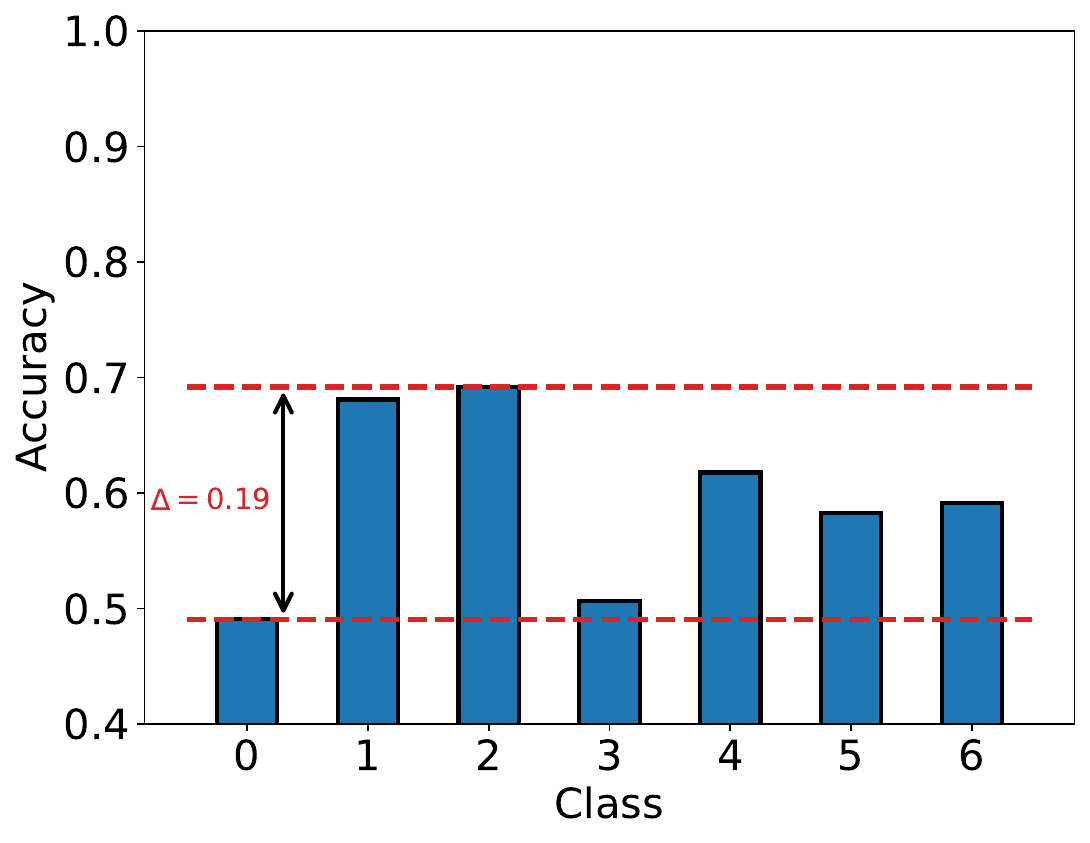}}
   \subfigure[GGD on Cora]{\includegraphics[width=0.24\textwidth]{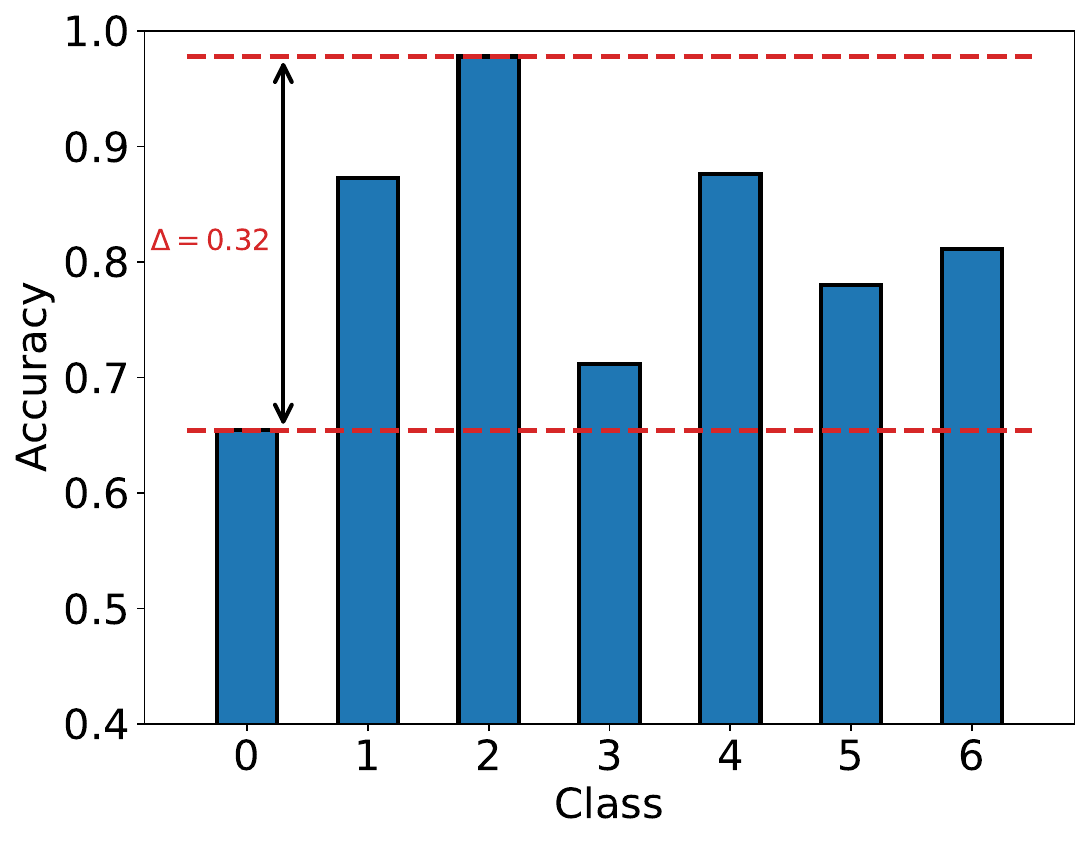}}
    \subfigure[MLP on Citeseer]{\includegraphics[width=0.24\textwidth]{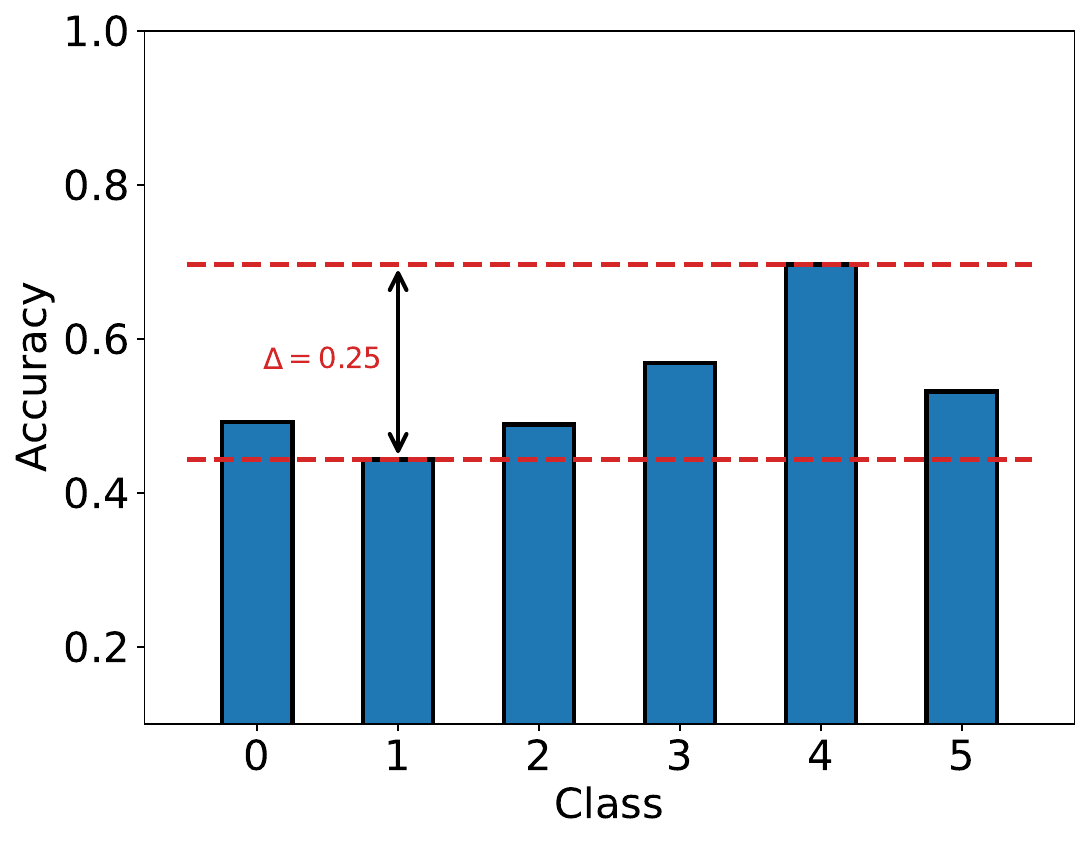}}
  \subfigure[GGD on Citeseer]{\includegraphics[width=0.24\textwidth]{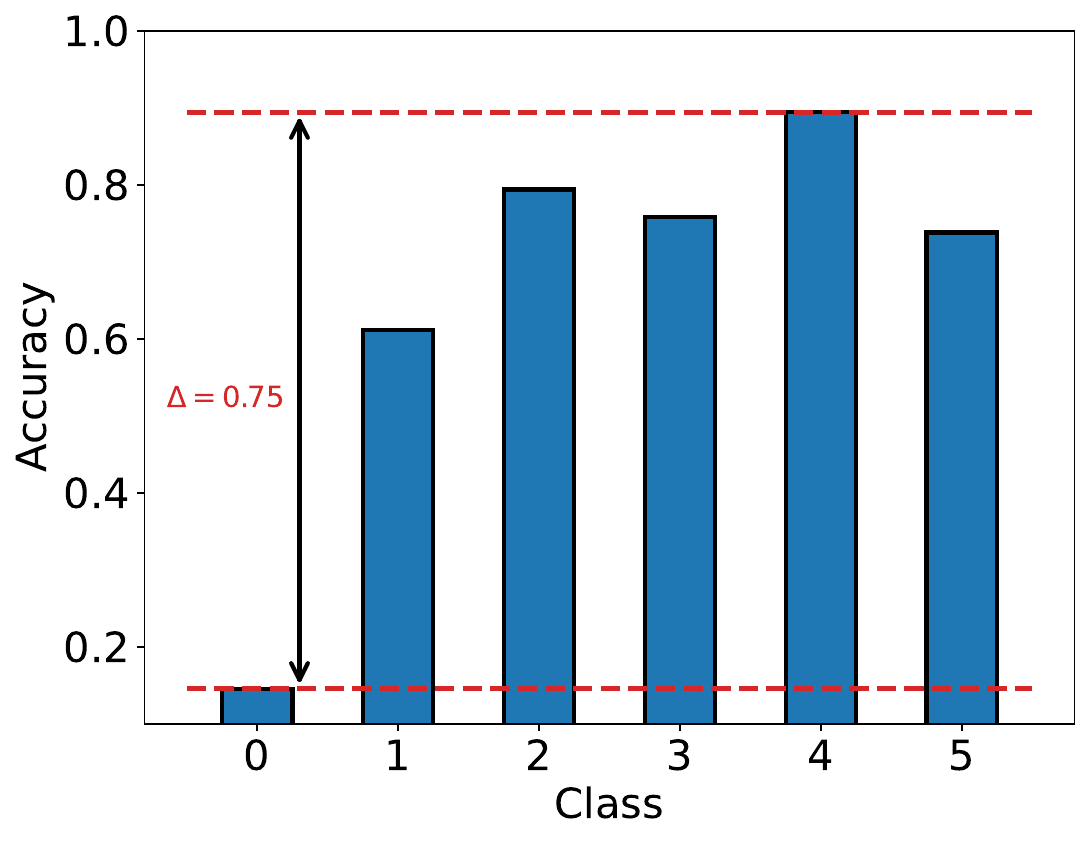}}
	\caption{The classification performance of MLP and GGD in different classes on Cora and Citeseer, where each class has 20 labeled nodes in the training set. $\Delta$ represents the maximum performance difference observed between classes. }
 	\label{fig:exp}
\end{figure*}

Based on the theoretical analysis, we then focus on how to alleviate the  community bias. We propose a simple data augmentation technique, namely random graph coarsening, and provide theoretical justifications on the effectiveness. We finally propose a novel graph contrastive learning model, called \ourmodel, which utilizes random graph coarsening as data augmentation and uses a contrastive loss that compares the coarsened graph with the original graph. Empirical results on real datasets show our model effectively reduce performance disparities between different classes and also achieves better overall accuracy than baselines, confirming our theoretical analyses.

Our contributions are summarized as follows: 
\begin{enumerate}

\item We uncover the community bias in the graph and analyze the causes of this problem from a spectral perspective. We show that local structural bias leads to embedding density imbalance, which is harmful on downstream tasks in terms of fairness.

\item We show that an appropriately designed random graph coarsening algorithm can be used as an effective data augmentation tool for alleviating the issue of embedding density imbalance.

\item Based on our theoretical analysis, we propose a novel GCL model, called \ourmodel. Our model mitigates community bias by comparing the coarsened graph with the original graph.

\item We empirically compare \ourmodel~with other graph contrastive learning models in various datasets. Experimental results demonstrate the advantage of \ourmodel, which confirms the effectiveness of using random coarsening to mitigate community bias.

\end{enumerate}

%% file: preliminaries.tex
\section{Preliminaries}
\noindent\textbf{Notation.} Consider an undirected graph $G=(V, E, X)$, where $V$ represents the vertex set, $E$ denotes the edge set, and $X\in \real^{n\times \mathcal{D}}$ is the feature matrix. Let $n=|V|$ and $m=|E|$ represent the number of vertices and edges, respectively. We use $A\in\{0,1\}^{n\times n}$ to denote the adjacency matrix of $G$ and $\{v_i, v_j\}$ to denote the undirected edge between node $v_i$ and node $v_j$. The degree of node $v_i$ denoted as $d_i$ is the number of edges incident on $v_i$. The degree matrix $D$ is a diagonal matrix and its $i$-th diagonal entry is $d_i$. 

\noindent\textbf{Graph neural network}.In each layer of a GNN, the representation of a node is computed by recursively aggregating and transforming representation vectors of its neighboring nodes from the last layer. One special case is the Graph Convolutional Network (GCN) \citep{kipf2016semi}. The layer-wise propagation rule of GCN is:
\begin{equation}
	H^{(l+1)}=\sigma \left( \widetilde{D}^{-\frac{1}{2}}\widetilde{A}\widetilde{D}^{-\frac{1}{2}}H^{(l)}W^{(l)} \right),
\end{equation}
where $\widetilde{A} = A+I$, $\widetilde{D}=D+I$ and $W^{(l)}$ is a learnable parameter matrix. GCNs consist of multiple convolution layers of the above form, with each layer followed by an activation $\sigma$ such as Relu. 

\noindent\textbf{Graph coarsening}. The coarse graph is a smaller graph $G'=(A', X')$. $G'$ is obtained from the original graph by computing a partition $P$ of $V$. The partition can be represented by a binary matrix $P\in\{0,1\}^{n\times n'}$, with $P_{ij} = 1$ if and only if vertex $i$ belongs to cluster $j$.  We define $S$ as the set of super-node and for each super-node $i \in S$, $S_i$ as the set of nodes that make up the super-node $i$. $I_u$ is the index of which supernode node $u$ belongs to.

%% file: convergence2.tex
\section{Exploring Community Bias}
There has been lots of work investigating the asymptotic behavior of GNNs as the number of layers $L$ goes to infinity, e.g., \citep{oono2020asymptotic,rong2019dropedge}. The general conclusion is that as $L\rightarrow \infty$, the representations of all nodes converge to a $1$-dimensional subspace, assuming the graph $G$ is connected. The convergence speed is determined by the second largest eigenvalue of the message passing operator $\hat{A}$. Following the notation from \cite{oono2020asymptotic}, we denote the maximum singular value of $W^{(l)}$ by $\omega_l$ and set $\omega:= \max_{l\in [L]} \omega_l$ and assume that $W^{(l)}$ of all layers are initialized so that $\omega \leq 1$. Given a subspace $\mathcal{M}$, we use $d_\mathcal{M}:= \mathbf{inf}_{Y\in \mathcal{M}}||X - Y||_F$ to measure the closeness between $X$ and $\mathcal{M}$, where $||\cdot ||_F$ denote the Frobenius norm. \cite{oono2020asymptotic} shows that if $G$ is connected, there is a 1-d subspace $\mathcal{M}$ such that for all $l$
\begin{align}
    d_\mathcal{M}(H^{(l+1)}) \leq \omega\lambda d_\mathcal{M}(H^{(l)}), \label{eqn:convergence}
\end{align}
which means the embeddings of all nodes collapse to $\mathcal{M}$ exponentially fast. However, in real applications, $L$ is typically small. In these cases, such asymptotic results are not accurate predictions of the model behavior. In particular, they ignore structural differences between different regions of the graph. Here, we illustrate the problem through a simple example. 

\subsection{Illustration of Embedding Density Imbalance}\label{sec:convergenceBias}
We consider the following example. There are two communities $C_1, C_2$ in the graph, and $C_1$ is more densely connected than $C_2$. $C_1$ and $C_2$ are loosely connected (see the left of Figure~\ref{fig:convergence}). Then, the expression of the symmetric normalized adjacency matrix $\hat{A}$ can be succinctly represented using a block matrix as follows:
$
\begin{bmatrix}
  \hat{A_1} & \hat{B_1}\\
 \hat{B_2} & \hat{A_2}   
\end{bmatrix}
$, where the matrix is partitioned according to $C_1$ and $C_2$. By the assumption, $||\hat{B_1}||_F$ and $||\hat{B_2}||_F$ are close to $0$, and since $C_1$ has a better connectivity than $C_2$, the second largest eigenvalue of $\hat{A_1}$, denoted by $\lambda(\hat{A_1})$ is smaller than $\lambda(\hat{A_2})$ \citep{chung1997spectral}. According to (\ref{eqn:convergence}), the representations of all nodes converges to $\mathcal{M}$ with speed exponential in $\lambda(\hat{A})$. A more detailed analysis presented below shows that, the two communities exhibit different convergence speed in the first few layers, due to different local connectivities. 

The representations of the nodes in the two communities can be separately expressed as $H_1^{(l+1)}=\sigma ( \hat{A_1}H_1^{(l)}W^{(l)} + \hat{B_1}H_2^{(l)}W^{(l)} )$ and $H_2^{(l+1)}=\sigma ( \hat{A_2}H_2^{(l)}W^{(l)} + \hat{B_2}H_1^{(l)}W^{(l)} )$. Since operator norms of $\hat{B_1}$ and $\hat{B_2}$ are very close to $0$, based on the result of \cite{oono2020asymptotic}, we have:
\begin{equation}
\left\{
\begin{array}{cc}
 d_\mathcal{M}(H^{(l+1)}_{1}) \approx d_\mathcal{M} (\sigma ( \hat{A_1}H_1^{(l)}W^{(l)} ))  \leq \omega\lambda(\hat{A_1}) d_\mathcal{M}(H^{(l)}_{1}), \\
  d_\mathcal{M}(H^{(l+1)}_{2}) \approx d_\mathcal{M} (\sigma ( \hat{A_2}H_2^{(l)}W^{(l)} )) \leq \omega\lambda(\hat{A_2}) d_\mathcal{M}(H^{(l)}_{2}), \label{eqn:localCon}
\end{array}
\right.
\end{equation}

For large $l$, the effect of $\hat{B_1}$ and $\hat{B_2}$ cannot be ignored, and eventually the convergence follows (\ref{eqn:convergence}). However, when the number of iterations is relatively small as in many real applications, the embedding distributions of different regions are much better characterized by local connectivity (\ref{eqn:localCon}).

If the structure of communities differs dramatically, for example, $\lambda(\hat{A_1}) \ll \lambda(\hat{A_2})$, the embeddings of nodes in $C_1$ will be much more concentrated than those in $C_2$. Figure \ref{fig:convergence} provides an example of this phenomenon. Although the node features for $C_1$ and $C_2$ are sampled with the same variance, the variance of their embeddings differs due to the convergence bias resulting from the distinct structures.
We provide a more quantitative analysis on the contextual stochastic block model (CSBM) \citep{deshpande2018contextual}, a widely used statistical model for analyzing expressive power of GNNs \citep{baranwal2021graph,wu2022non}; please refer to Appendix \ref{appendix:csbm} for more details. 
\begin{figure}[thbp]
	\centering
\includegraphics[width=0.7\textwidth]{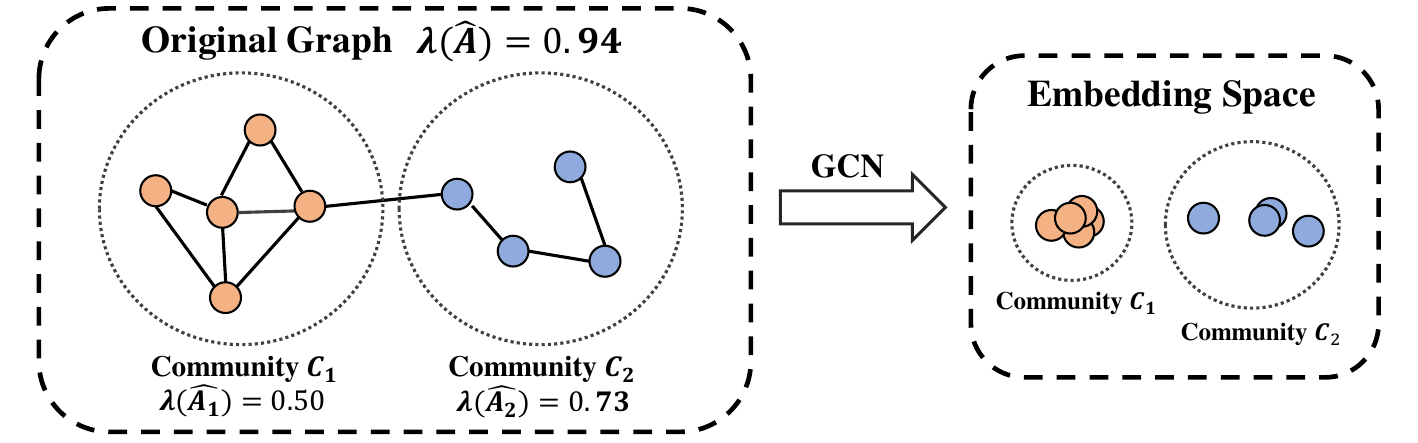}
\vspace{-3mm}
 \caption{A simple example of convergence bias. We sample feature vectors for nodes in $C_1$ and $C_2$ from normal distributions, $\mathcal{N}(
\begin{bmatrix}
-1 \\
-1
\end{bmatrix}
,\begin{bmatrix}
1&0 \\
0&1
\end{bmatrix})$ and $\mathcal{N}(
\begin{bmatrix}
1 \\
1
\end{bmatrix}
,\begin{bmatrix}
1&0 \\
0&1
\end{bmatrix})$, respectively. We then apply a 2-layer graph convolutional network, and the resulting embeddings of the nodes are visualized.}
	\label{fig:convergence}
\end{figure}

\subsection{From Embedding Density Imbalance to Community Bias}\label{sec:accImba}
We have discussed how different local convergence speeds lead to varying degrees of dispersion in local embedding distributions, and next we provide a theoretical justification on why this can lead to community bias in downstream tasks. 

To illustrate the issue, we consider a binary classification problem and assume the node embeddings from each class follows a Gaussian distribution. 
We consider the optimal Bayes classifier for the above model, which is known to be the \emph{quadratic discriminant} rule. 
\begin{definition}[Quadratic Discriminant Analysis]
    For a binary classification problem, where class 1 has mean ${\mu}_1$ and variance ${\Sigma}_1$ and class 2 has mean ${\mu}_2$ and variance ${\Sigma}_2$. Each sample is drawn from either one with equal probability. Given a new sample $x$, the QDA rule is
    \begin{equation}
    \arg\max_{c=1,2} -\frac{1}{2}\log\det {\Sigma}_c-\frac{1}{2}(x-{\mu}_c)^T{\Sigma}^{-1}_c(x-{\mu}_c).
    \end{equation}
\end{definition}
To simplify the discussion, we focus on the one-dimensional case, while the results for higher dimensions are similar. More specifically, each sample $x$ is drawn from $\mathcal{N}(\mu_1, \sigma_1^2)$ or $\mathcal{N}(\mu_2, \sigma_2^2)$ with equal probability, and let $y$ be the label of $x$. Assume $\sigma_1 \neq \sigma_2$, meaning the data distribution of one class is more concentrated than the other, which models the embedding density imbalance. For this simple case, the optimal classification error can be represented as a closed form.
\begin{proposition}
For the above QDA classifier, the error probability of samples from class 1 is:
\begin{equation}
    p_1= \mathbb{P}(Y^2>\frac{(\sigma_1^2+\sigma_2^2)^2}{\sigma_2^2-\sigma_1^2}-(\sigma_2^2-\sigma_1^2)+2\sigma_1^2\sigma_2^2\mathrm{log}(\frac{\sigma_2}{\sigma_1})),
\end{equation}
    
    where $Y\sim \mathcal{N}(\sqrt{|\sigma_2^2-\sigma_1^2|}+(2I(\sigma_1>\sigma_2)-1)\frac{\sigma_1^2+\sigma_2^2}{\sqrt{|\sigma_2^2-\sigma_1^2|}}, |\sigma_1^2-\sigma_2^2|\sigma_1^2)$.
\end{proposition}

\begin{wrapfigure}{r}{0.31\textwidth}
\vspace{-4.5mm}
    \centering
\includegraphics[width=0.8\linewidth]{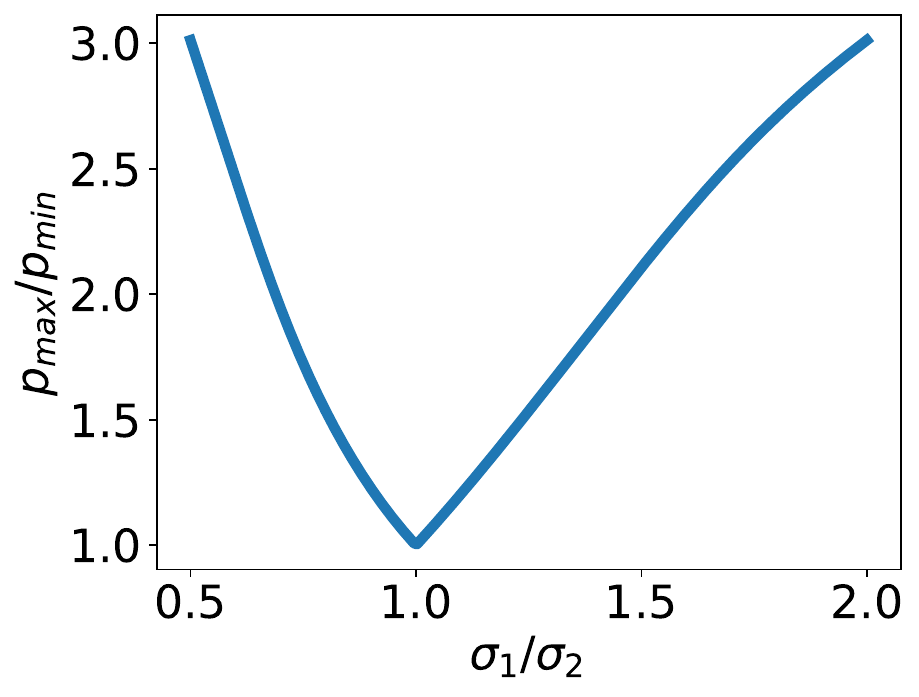}
\vspace{-3mm}
\caption{Fairness $\kappa$ vs.\ ratio between variances.} \label{fig:var_imb}
\vspace{-5mm}
\end{wrapfigure}
The classification error probability of class 2 is symmetric and thus omitted.
The above expression is quite complicated. To demonstrate the community bias issue, we define the classification fairness as $\kappa = \frac{\max\{p_1,p_2\} }{\min\{p_1,p_2\}}$ (larger $\kappa$ means more severe fairness issue), and investigate how the imbalance in data distribution affects $\kappa$. To this end, we fix the value of $\sigma_1^2+\sigma_2^2$, and vary the ratio $\sigma_1/\sigma_2$. The value of $\kappa$ corresponding to different ratios is plotted in Figure~\ref{fig:var_imb}. It is clear that the unfairness of the classifier is more severe when the variances of different classes are more imbalanced.

\section{Mitigating community Bias via Random Coarsening} \label{sec:Coarsening}
In this section, we present the main idea and theoretical justifications of our method to alleviate the community bias problem discussed in the last section. 
Our goal is to make the embedding distribution of sparse classes more concentrated. For this to happen, we first randomly partition the graph into clusters $S=\{S_1,\cdots, S_t\}$ according to some random process. Let $f$ be a GNN encoder, and $f(u)$ be the embedding of node $u$. We define the embedding for a cluster $S_i$ as $f(S_i)=\frac{1}{|S_i|}\sum_{v\in S_i}f(v)$. We use the loss
\begin{align}\label{eqn:concentrate}
    \sum_{S_i\in S}\sum_{u\in S_i}\Vert f(u)-f(S_i)\Vert^2
\end{align}
to regularize the GNN encoder, which encourages nodes in each cluster in the random partition to be more concentrated. In the following, we first show that if the distribution of the random partition, denoted by $\mathcal{P}$, satisfies certain requirements, the above loss has the implicit effect of pushing the embeddings of sparse classes more heavily. Then, we provide a specific random partition algorithm, namely random graph coarsening, which meets the requirements.

Following the setting in Section~\ref{sec:accImba}, we consider a binary classification problem. Assume that for a random partition drawn from $\mathcal{P}$, the probability that two nodes from class 1 (class 2) lie in the same cluster is $q_1$ ($q_2$), and the probability that two nodes from different classes are clustered together is $q_{12}$. 
We have the following lemma, the proof of which is provided in the appendix \ref{appendix:a}.

\begin{lemma}\label{lem:weightedReg}
Let $C_1$ and $C_2$ be the two classes of nodes. Suppose each cluster has the same size $s$, then
\begin{align*}
    &\mathbb{E}_{P\sim \mathcal{P}}\sum_{S_i\in S}\sum_{u\in S_i}s\Vert f(u)-f(S_i)\Vert^2 \\
    =&q_1\sum_{u,v\in C_1,u\neq v}\Vert f(u)-f(v)\Vert^2+q_2\sum_{u,v\in C_2,u\neq v}\Vert f(u)-f(v)\Vert^2 + q_{12}\sum_{u\in C_1,v\in C_2}\Vert f(u)-f(v)\Vert^2.
\end{align*}
\end{lemma}

Now suppose $C_1$ is denser than $C_2$, and by the analysis in Section~\ref{sec:convergenceBias}, the embeddings in $C_1$ are likely to be more concentrated. To make the embedding densities of $C_1$ and $C_2$ more balanced, from Lemma~\ref{lem:weightedReg}, a preferred random partitioning algorithm should satisfy 
\begin{align}
    q_2 > q_1 > q_{12}. \label{eqn:q2-q1-q12}
\end{align}

Here, we design a random graph coarsening strategy to obtain such a reasonable partition, which can satisfy (\ref{eqn:q2-q1-q12}). Due to the homophily principle that similar nodes may be more likely to attach to each other than dissimilar ones, if we always merge nodes that are connected by an edge, $q_{12}$ should be less than $q_1$ and $q_2$. 
Therefore, our random graph coarsening strategy iteratively merging two (super) nodes through edge contraction to form more super nodes. Eventually, we can obtain a coarsened graph, where each supernode represents a cluster, and the nodes that constitute this supernode belong to the same partition. By utilizing the original graph and the coarsened graph, we can optimize loss (\ref{eqn:concentrate}) to mitigate community bias.

In our random graph coarsening, we contract each edge $(u,v)$ with probability proportional to some weight $\omega(u,v)$. Since on average, the degrees of nodes in $C_1$ are higher than $C_2$, to realize $q_2>q_1$, we adopted a simple weight function $\omega(u,v) = \frac{1}{d_u + d_v}$ for random edge selection during coarsening, ensuring a higher probability for low-degree edges to participate in the random coarsening process. Moreover, to prevent the formation of large supernodes, we use a threshold limiting the size of supernodes during the coarsening process. 
A detailed description of our random coarsening algorithm is provided in the appendix \ref{sec:algorithm}.

%% file: method.tex
\section{Random Graph Coarsening Contrastive Learning}
Based on the theoretical insights from Section~\ref{sec:Coarsening}, we present a novel multi-view graph contrastive learning method \ourmodel, which can effectively alleviate the issue of community bias. Compared to existing GCL models, the key difference in our approach is to use random graph coarsening as the graph augmentation method and a specially designed loss function for this framework.

\subsection{Framework of \ourmodel}

The GCL method generates different views through graph augmentation, and then trains the model parameters by comparing the node embeddings from different views. The architecture of \ourmodel is presented in Figure \ref{fig:rgccl}.  In our \ourmodel, we regard the original graph as one view and the coarsened graph as another view, and then compare the corresponding nodes in these two views. In order to alleviate community bias, we need to push the node embedding $f(u)$ towards its corresponding cluster center embedding $f(S_u)$ according to the analysis in Section~\ref{sec:Coarsening}. Specifically, the super-node embedding computed in the coarsened view is used as the cluster center, and then each node embedding in the original graph and its corresponding super-node embedding in the coarsened graph is defined as a positive pair. By conducting contrastive learning on such positive pairs, we make the embedding distribution of sparse classes more concentrated, thereby reducing community bias.

\begin{figure}[t]
	\centering
\includegraphics[width=0.9\textwidth]{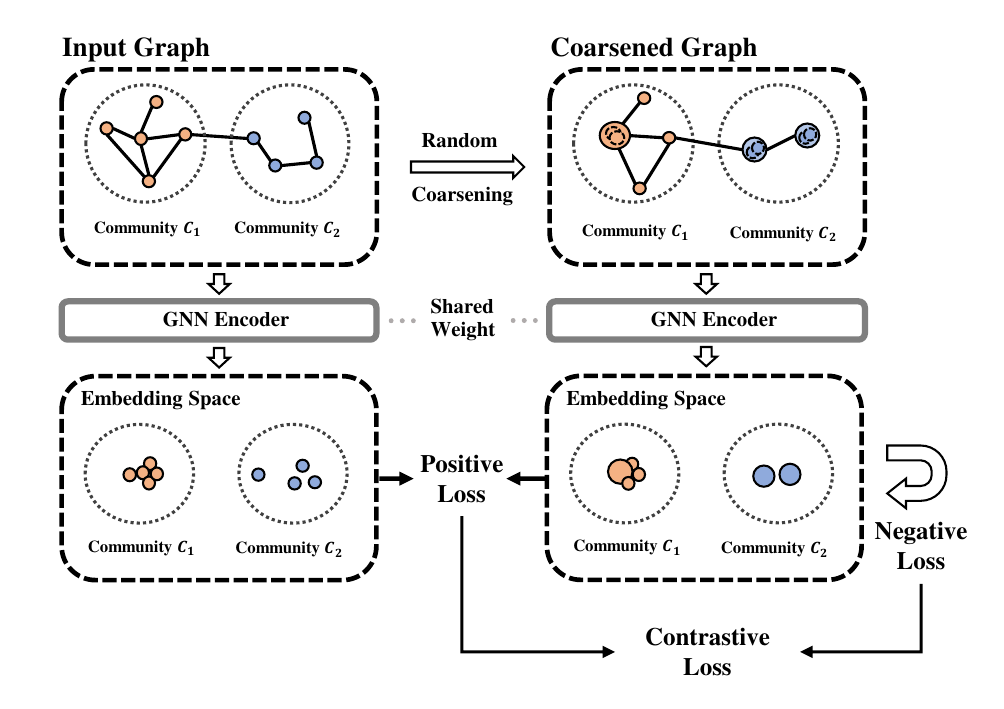}
\vspace{-6mm}
	\caption{The architecture of the proposed \ourmodel.}
\label{fig:rgccl}
\vspace{-3.0mm}
\end{figure}

In this work, we use a method different from the previous graph coarsening algorithms \citep{huang2021scaling} to construct the coarsened feature matrix.  A difference is that they construct the new feature matrix simply by summing i.e., given a partition matrix $P$, $X' =P^TX$. Here, a normalization based on degrees is applied: $X'=\widetilde{D}^{'-1}P^{T}\widetilde{D}X$, where $\widetilde{D}'$ and $\widetilde{D}$ are the degree matrices of the two views. The purpose is to ensure, as the number of graph propagations goes to infinity, the embedding of a node and its corresponding supernode in the coarsened view converge to the same point. More discussions are included in the appendix \ref{sec:error_bound}. On the other hand, graph coarsening is a method of data reduction, so using random graph coarsening as data augmentation can also reduce the resource consumption of GNN training.

Specifically, we apply the random graph coarsening algorithm to generate one graph augmentation $G'=(A',X')$ in each epoch during training. Then, we compute the coarsened graph and original graph embeddings by a GNN encoder with shared parameters: $H=\mathrm{GNN}(A', X', \theta)$ and $Z=\mathrm{GNN}(A, X, \theta)$. 

Recall that positive pairs are of the form $(u, S_u)$, where $S_u$ denotes the corresponding supernode of $u$. So we penalize $\|Z_u - H_{S_u}\|_F^2$. The loss function of positive pairs can be described more concisely in the matrix form. 
Let $Z'=P H$, then $Z'_u = H_{S_u}$. Therefore, the positive pair loss function is
\begin{equation}
\mathcal{L}_{pos}=\|Z-Z'\|^2_F=\|Z\|_F^2+\|Z'\|_F^2-2\mathsf{Tr}(Z^T Z').
\end{equation}
If $Z$ and $Z'$ are normalized appropriately, we only need to minimize $-2\mathsf{Tr}(Z^T Z')$.

On the other hand, nodes embedded from the original graph $G$ into $Z$ may encounter imbalance issues. Since we prioritize low-degree edges in the coarsening algorithm, the coarsened graph will typically have a more balanced embedding distribution. Therefore,  we do not pick negative pairs from the original graph and only compute a negative pair loss with the coarsened graph. There are various methods for selecting negative pairs and computing the loss; here we use the loss from \cite{zhang2020sce}, which is derived from the graph partition problem. More specifically, we randomly sample a small set of supernode pairs $\mathcal{N'} \subset V'\times V'$, and the negative pair loss function is:
\begin{equation}
\mathcal{L}_{neg}=\frac{\alpha}{\sum_{(i, j)\in \mathcal{N'}} n_i n_j\|h_i-h_j\|^2}.
\end{equation}
where $h_i$ and $h_j$ are the embeddings of supernodes $i$ and $j$, and $n_i$ and $n_j$ are the number of original nodes contained in supernodes $i$ and $j$.

Optimizing $\mathcal{L}_{neg}+\mathcal{L}_{pos}$ will be difficult due to the huge difference of the scale of $\mathcal{L}_{pos}$ and $\mathcal{L}_{neg}$. Therefore, we transform $\mathcal{L}_{pos}$ into the following form
\begin{equation}
\mathcal{L}_{pos}=\frac{\beta}{\mathsf{Tr}(Z^TZ')}.
\end{equation}

Finally, the loss function of our model is

\begin{equation}
    \mathcal{L} = \frac{\alpha}{\sum_{(i, j)\in \mathcal{N'}} n_i n_j\|h_i-h_j\|^2} + \frac{\beta}{\mathsf{Tr}(Z^TZ')} .
\end{equation}

%% file: divergence3.tex
\subsection{Generalizability of \ourmodel}
The generalizability of self-supervised learning methods has recently been theoretically analyzed in \citep{huang2021towards, wang2022grade}. They characterize it with three properties, namely the concentration of augmented data, the alignment of positive samples and the divergence of class centers. \cite{huang2021towards} also shows that the divergence of class centers is controlled by classic contrastive losses such as InfoNCE and the cross-correlation loss.

Following the proof of \cite{huang2021towards}, we show that our self-supervised loss $\mathcal{L}_{neg}$ can also upper bounds the divergence of class centers, thus classes will be more separable if our objective is optimized. We also investigate the concentration properties of our random coarsening data augmentation. The details are provided in the appendix \ref{appendix:d} due to the space constraints.

%% file: related.tex
\section{Related Work}

\noindent\paragraph{Contrastive learning on graphs} Contrastive learning is a type of unsupervised learning technique that learns a representation of data by differentiating similar and dissimilar samples. 
It has been used in a variety of applications within the domain of graph data. DGI \citep{velivckovic2018deep} and MVGRL \citep{hassani2020contrastive} contrast node embeddings with graph embeddings using a loss function based on mutual information estimation \citep{belghazi18mutual, hjelm2018learning}. GRACE \citep{zhu2020grace} and its variants \citep{zhu2021gca, you2020graph} aim to maximize the similarity of positive pairs and minimize the similarity of negative pairs in augmented graphs in order to learn node embeddings. To counter the performance degradation induced by false negative pairs, CCA-SSG \citep{zhang2021ccassg} simplifies the loss function by eliminating negative pairs. In order to  reduce the computational complexity of contrastive loss, GGD \citep{zheng2022ggd} discriminates between two groups of node samples using binary cross-entropy loss. For a broader understanding, we recommend that readers refer to the latest surveys \citep{liu2022survey, xie2023review}.

\noindent\paragraph{Structural bias on graphs} Fair graph mining has attracted much more research attention since recent studies reveal that there are unfairness in a large number of graph mining models. Several notions of fairness have been proposed in recent survey \citep{dong2022fairness}, and structural bias is mainly manifested as degree bias. Prior studies \citep{feng2020invest, kang2022rawlsgcn, dong2022edits, liu2023on} have primarily concentrated on degree bias in supervised graph learning. 
GRADE \citep{wang2022grade} proposes a graph data augmentation method to mitigate the degree bias issue in unsupervised scenarios.

\noindent\paragraph{Graph coarsening} Recently, graph coarsening techniques have been used to address issues in graph neural networks \citep{fahrbach2020faster,deng2019graphzoom,huang2021scaling, jin2022conden}. Graph coarsening with spectral approximation guarantees are studied in \citep{li2018spectral,loukas2019graph,jin2020graph}. Graph coarsening can reduce the size of graph by combining the similar nodes, then the coarsened graph can be used for downstream tasks related to the graph. Existing graph coarsening techniques primarily strive to maintain the overall graph structure, resulting in a static and downsized coarsened graph. These methods overlooks the local structural bias, and typically involves massive computational costs.

%% file: exp.tex
\section{Experiments}

\subsection{Experimental Setup}

\noindent\paragraph{Datasets} The results are evaluated on six real-world datasets \citep{kipf2016semi, velivckovic2018deep, zhu2021gca, hu2020ogb}, Cora, Citeseer, Pubmed, Amazon Computer, Amazon Photo, and Ogbn-Arixv. Graph representation learning shows different degrees of community bias in these datasets. More detailed statistics of the seven datasets are summarized in the appendix \ref{sec:details}. On small-scale datasets Cora, Citeseer, Pubmed, Photo and Computers, performance is evaluated on random splits. We select 20 labeled nodes per class for training, while the remaining nodes are used for testing. All results on small-scale datasets are averaged over 50 runs, and standard deviations are reported.  For Ogbn-Arixv, we use fixed data splits as in previous studies \cite{hu2020ogb}.

\noindent\paragraph{Baselines} We compare our approach against nine representative graph embedding models:  Deepwalk \citep{Perozzi:2014:DOL:2623330.2623732}, DGI \citep{velivckovic2018deep}, GraphCL \citep{you2020graph}, GRACE \citep{zhu2020grace},  GCA \citep{zhu2021gca}, CCA-SSG \citep{zhang2021ccassg}, gCool \citep{li2022gcool}, GRADE\citep{wang2022grade} and GGD\citep{zheng2022ggd}. 
For all the baselines, we use the public code released in their previous papers. All models evaluate the learned representations by training and testing classifiers with the same settings.

\noindent\paragraph{Implementation details} 
For our model, we use the Adam optimizer and all embedding dimensions are set to 512. For the graph coarsening operation, the coarsening rates are set to $[0.3, 0.5,0.5,0.7,0.7,0.7]$ respectively and the threshold for the supernode size is set to 10 across all datasets. More experimental details are listed in the appendix \ref{sec:details}.

\subsection{Results and Analysis}

\begin{table}[thbp]\tiny
	\caption{Summary of results in terms of mean node classification accuarcy and standard deviation over 50 runs on five datasets. The training set contains 20 labeled nodes per class. The highest accuracy in each column is highlighted in bold and the runner ups are underlined.}	\label{tab:node}
	\begin{tabular}{lcccccccccc}\toprule
		\multicolumn{1}{c}{\multirow{2}*{\textbf{Method}}} & \multicolumn{2}{c}{\textbf{Cora}}& \multicolumn{2}{c}{\textbf{Citeseer}}& \multicolumn{2}{c}{\textbf{Pubmed}}& \multicolumn{2}{c}{\textbf{Photo}}& \multicolumn{2}{c}{\textbf{Computers}}\\ 
		\cmidrule(r){2-11}
		&Acc &Macro-F1&Acc &Macro-F1&Acc &Macro-F1&Acc&Macro-F1&Acc&Macro-F1\\ \midrule 
		\textbf{Deepwalk} & 67.2$\pm$1.7& 66.5$\pm$1.5&40.0$\pm$2.1 &38.3$\pm$2.0 & 66.9$\pm$2.8&65.6$\pm$2.7&85.1$\pm$1.2&83.9$\pm$1.2&77.3$\pm$1.6&77.2$\pm$1.5\\
		\textbf{DGI} &78.5$\pm$0.9 &77.2$\pm$0.9 &70.4$\pm$1.0 &63.6$\pm$1.4 & 72.5$\pm$3.3 &72.5$\pm$3.3&87.9$\pm$1.3&86.2$\pm$1.3&79.7$\pm$1.6&78.4$\pm$1.3\\
		\textbf{GraphCL} &78.3$\pm$1.5 &76.7$\pm$1.7 &70.6$\pm$1.2 &64.1$\pm$1.4 &71.7$\pm$3.5 &71.7$\pm$3.6&88.3$\pm$1.3&86.7$\pm$1.2&79.7$\pm$1.4&78.5$\pm$1.1\\ 
		\textbf{GRACE} & 74.4$\pm$2.0&72.5$\pm$2.0 &68.9$\pm$1.0 &61.2$\pm$1.1 & \underline{76.1}$\pm$2.8& \underline{75.9}$\pm$2.7&85.1$\pm$1.6&83.5$\pm$1.4&76.2$\pm$1.9&75.2$\pm$1.5\\
		\textbf{GCA} &78.6$\pm$1.2&77.2$\pm$1.2 &68.8$\pm$1.5 &65.3$\pm$1.4 &75.4$\pm$3.0 &75.5$\pm$2.9&87.8$\pm$1.2&86.2$\pm$1.3&79.1$\pm$2.4&77.9$\pm$2.0\\
		\textbf{CCA-SSG} & 79.2$\pm$1.4 &78.0$\pm$1.4 & \underline{71.8}$\pm$1.0&\underline{66.3}$\pm$1.1 & 76.0$\pm$2.8&75.8$\pm$2.7&\underline{88.7}$\pm$1.1&\underline{86.9}$\pm$3.2&\textbf{82.7}$\pm$1.0&76.9$\pm$3.7\\
		\textbf{gCooL} &78.5$\pm$1.3 &77.1$\pm$1.1 &68.6$\pm$1.4 &64.9$\pm$1.3&75.5$\pm$3.0 &75.3$\pm$2.9&87.9$\pm$1.3&85.9$\pm$1.4&79.8$\pm$1.7&78.1$\pm$1.3\\
		\textbf{GRADE} & 81.5$\pm$1.0&80.2$\pm$1.0 &67.6$\pm$1.5 &64.2$\pm$1.3 &74.5$\pm$2.7 &74.5$\pm$2.6 & 87.1$\pm$1.2 & 80.4$\pm$3.0&75.8$\pm$1.2 & 64.7$\pm$3.1\\
		\textbf{GGD} &\underline{81.9}$\pm$0.9 &\underline{80.5}$\pm$0.8 & 70.1$\pm$1.3&66.2$\pm$1.1 &74.7$\pm$3.2 &74.4$\pm$3.1&87.2$\pm$1.5&85.4$\pm$1.4&80.4$\pm$1.8&\underline{80.0}$\pm$1.2\\
		\midrule 
		\textbf{\ourmodel} & \textbf{83.1}$\pm$0.8 & \textbf{82.0}$\pm$0.8 &\textbf{72.4}$\pm$0.9 & \textbf{67.7}$\pm$0.8 & \textbf{77.3}$\pm$2.9  & \textbf{77.1}$\pm$2.7 & \textbf{89.6}$\pm$1.2&\textbf{88.2}$\pm$1.2&\underline{81.2}$\pm$1.8&\textbf{80.2}$\pm$1.2 \\
		\bottomrule
	\end{tabular}
\end{table}

Table \ref{tab:node} reports the node classification accuracy and standard deviation on small datasets. In these data splits, we employ Acc and Macro-F1 as metrics to assess the overall performance of models. If Macro-F1 significantly drops compared to Acc, it indicates that the model's performance is not balanced across different communities. It is evident that the performance of \ourmodel~outperforms other GCL models within the given experimental framework. 
Mostly, \ourmodel~ surpasses the runner-up by an advantage of 1\%-2\%. Notably, although \ourmodel's Acc score on the Computers dataset is slightly lower than that of CCA-SSG, our Macro-F1 score is significantly higher than that of CCA-SSG. This suggests that CCA-SSG has generated significant community bias on Computers. Clearly, our specially designed model can improve the overall classification performance effectively while resolving the issue of community bias.

\begin{figure}[ht]
\centering
\subfigure[Cora]{\includegraphics[width=0.32\textwidth]{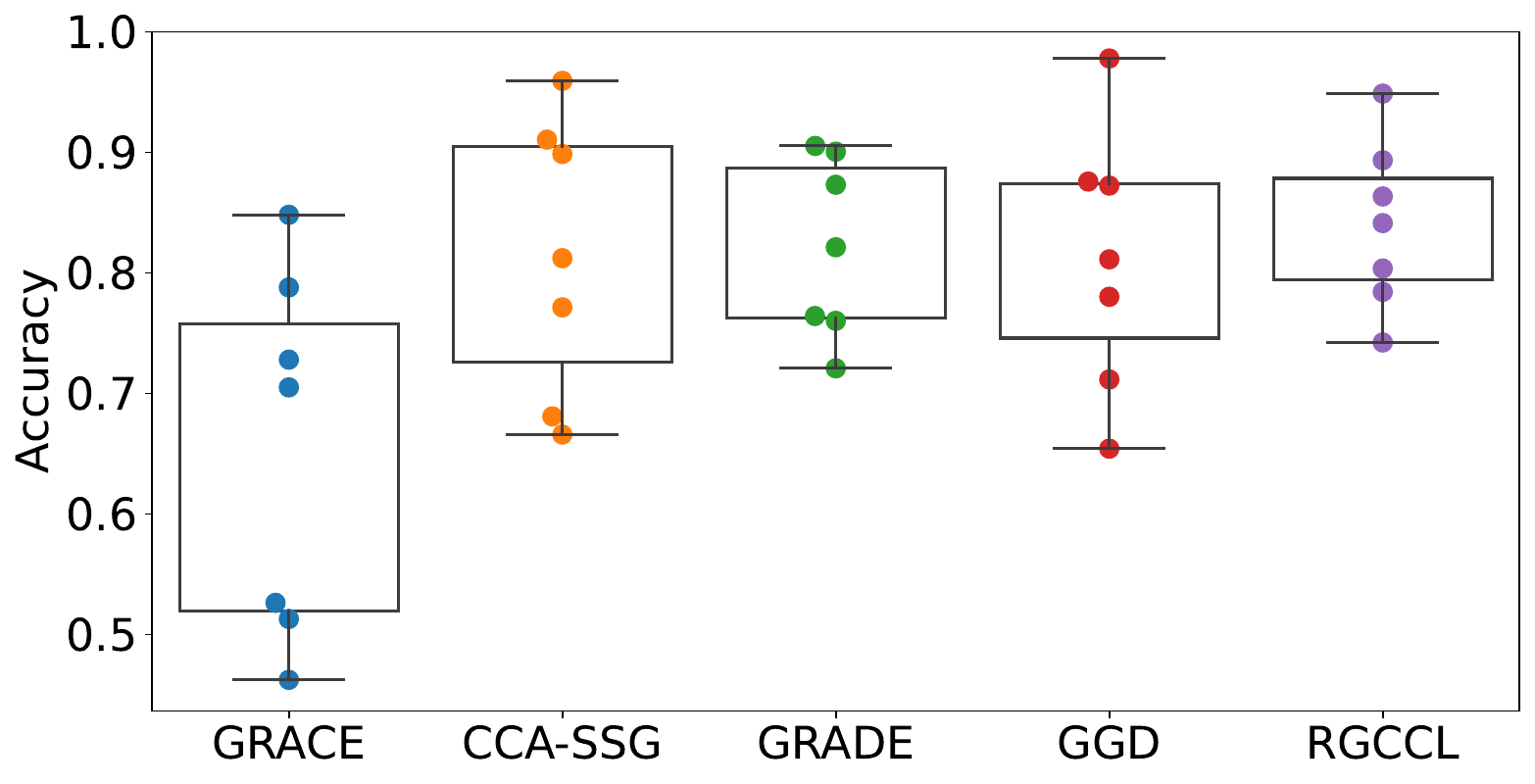}}
\subfigure[Citeseer]{\includegraphics[width=0.32\textwidth]{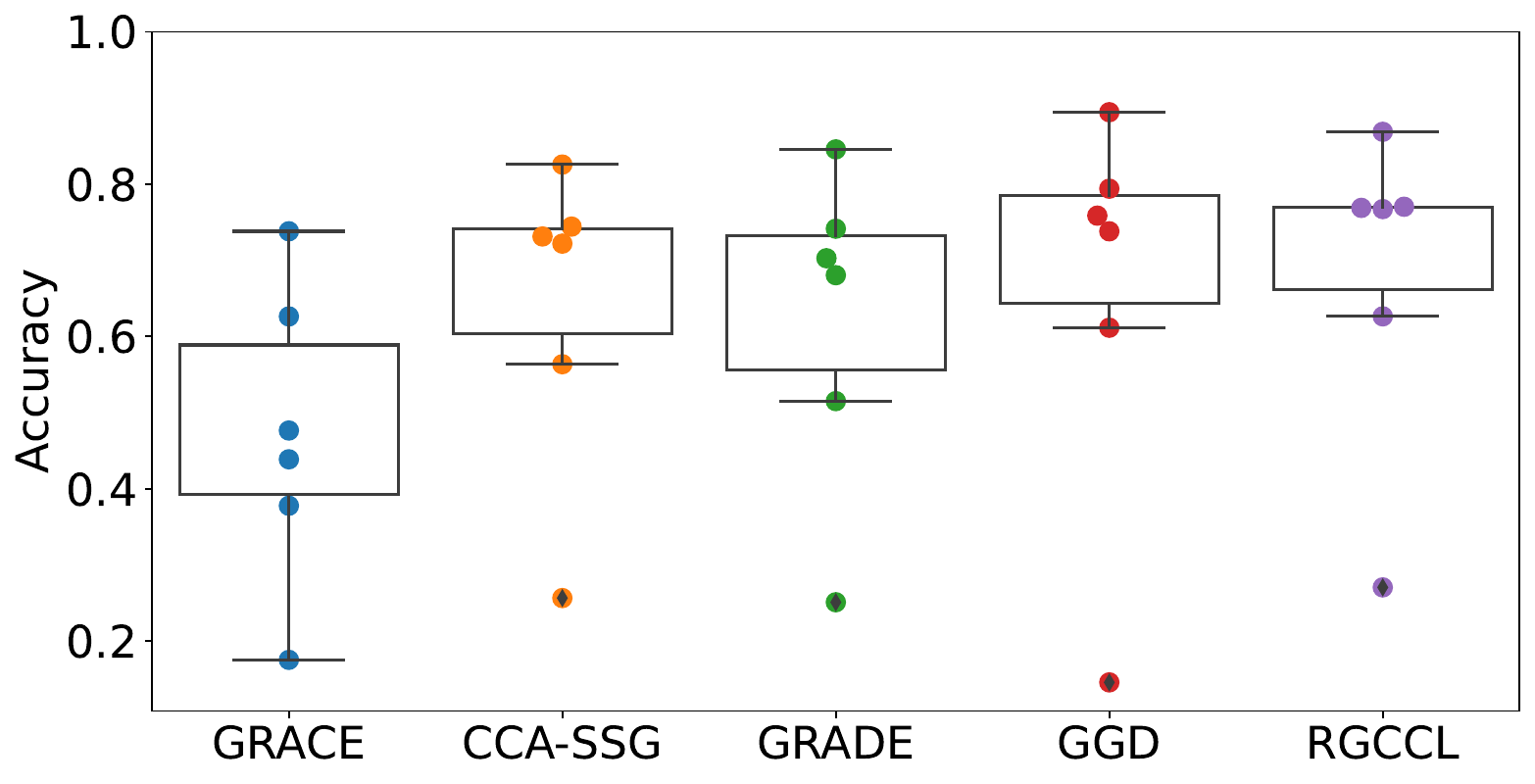}}
\subfigure[PubMed]{\includegraphics[width=0.32\textwidth]{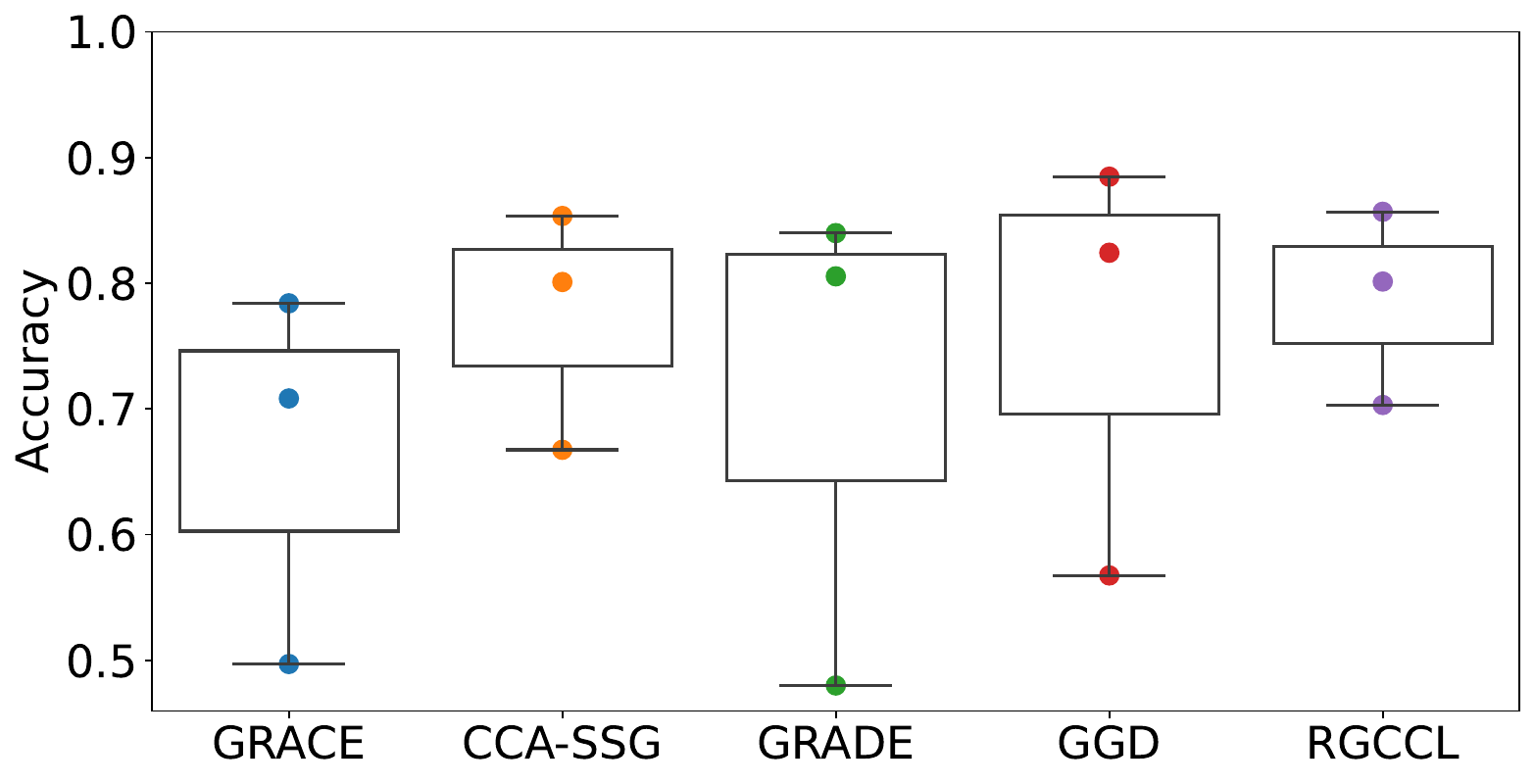}}
\vspace{-3mm}
\caption{Box plots of the average accuracy w.r.t. community for four baselines and RGCCL on the Cora, Citeseer and Pubmed dataset.}
\label{fig:box}
\end{figure}

\noindent\paragraph{Visualization} To further illustrate that RGCCL effectively mitigates the issue of community bias, we have visualized the performance of each model across different communities. For each community, we calculate the average performance of the model in this community and then draw a box plot based on these accuracies. This figure (Figure \ref{fig:box}) provides a visualization of the node classification accuracy in different communities on the Cora, Citeseer and PubMed. It is clear that our model has the smallest performance difference across different communities, while also having the best overall performance.

\begin{wraptable}{r}{0.5\textwidth}\scriptsize
\vspace{-6.5mm}
    \centering
\caption{The average and standard deviation of community density. A smaller average indicates higher embedding quality, while a smaller standard deviation suggests less community bias.}\label{tab:density}  
\begin{tabular}{lcccc}\toprule
\multicolumn{1}{c}{\multirow{2}*{\textbf{Method}}} &  \multicolumn{2}{c}{\textbf{Cora}}&  \multicolumn{2}{c}{\textbf{Citeseer}}\\ 
&Ave &Std&Ave &Std\\ \midrule 
    DGI &0.3782& 0.0294& 0.3402 & 0.0255\\
    GRACE & 0.2114 & 0.0252 & 0.1715 & 0.0163\\
    CCA-SSG & 0.2817 & 0.1031 &0.1672 & 0.0109  \\
    GRADE & 0.1983& 0.0352& 0.2141&0.0243\\
    GGD  &0.3183&0.0541&0.3297&0.0329\\
   	\midrule 
    \ourmodel&\textbf{0.0942} & \textbf{0.0084} & \textbf{0.1401}& \textbf{0.0097} \\
\bottomrule
\end{tabular}
\vspace{-3.0mm}
\end{wraptable}

\noindent\paragraph{Density} To further illustrate the high embedding quality of RGCCL, we present statistics of the learned representations and compare them with other popular GCL methods. Firstly, we measure the concentration of the embedding for each community by calculating their mean distances from the centroid (i.e., $V_C=\frac{1}{|C|}\sum_{i\in C}||z_i - \frac{1}{|C|} \sum_{i\in C} z_i||$ for community $C$). Next, we compute the average and standard deviation of $V_C$ across all communities. A smaller average value indicates that the representations for each community are more concentrated, which in turn makes the classification boundary easier to learn. A smaller standard deviation suggests a more balanced embedding density, which has been shown to be beneficial for classification fairness (as discussed in Section \ref{sec:accImba}). In Table \ref{tab:density}, we present the results of four baseline methods as well as RGCCL on Cora and Citeseer. RGCCL demonstrates not only more concentrated embeddings for each community but also the most balanced embedding density. These results strongly support our theory.

\begin{wraptable}{r}{0.35\textwidth}\scriptsize
\vspace{-6.5mm}
\centering
\caption{Summary of results in terms of accuracy on Ogbn-Arxiv.}\label{tab:large}  
\begin{tabular}{lcc}
\toprule
\multicolumn{1}{c}{\multirow{2}*{\textbf{Method}}} &  \multicolumn{2}{c}{\textbf{Ogbn-Arxiv}}\\ 
\cmidrule(r){2-3}
&Acc &Macro-F1\\ \midrule 
    DGI&68.8$\pm$0.2&46.8$\pm$0.4 \\
    GRACE&68.4$\pm$0.1&46.1$\pm$0.3  \\
    CCA-SSG&69.8$\pm$0.2&46.5$\pm$0.6 \\
    GRADE &67.7$\pm$0.2&45.0$\pm$0.4 \\
    GGD & 70.7$\pm$0.3&48.5$\pm$0.4 \\
   	\midrule 
\ourmodel&\textbf{71.7}$\pm$0.1&\textbf{50.6}$\pm$0.2\\
\bottomrule
\end{tabular}
\vspace{-6.5mm}
\end{wraptable}

\noindent\paragraph{Scalability} Another benefit of \ourmodel~is its small memory usage. This efficiency is primarily because the random graph coarsening is preprocessed on the CPU, resulting in a coarsened graph notably smaller than the original. Experiments were also conducted on the Arxiv dataset, which is a large-scale dataset for most GRL models. Larger size makes sub-sampling necessary for training on some baselines, but our RGCCL can be trained directly on the full graph. As shown in Table \ref{tab:large}, the Acc and Macro-F1 of \ourmodel~both exceed those of other baselines. For reports on memory usage, please refer to the Appendix \ref{sec:details}.

\noindent\paragraph{Effectiveness of different coarsening ratios}
We studied the effect of random coarsening ratio on model performance. The random coarsening ratio refers to the proportion of the number of nodes reduced relative to the total number of nodes. Table \ref{tab:ratio} shows the results of different coarsening ratios. According to our observations, the changes in the coarsening rate have a more significant impact on the Macro-F1. The model performs better when the coarsening rate is between 30\% and 50\%. This is consistent with the conclusions of our theoretical analysis.

\begin{table}[!htbp]\small
\setlength{\abovecaptionskip}{-0.2cm}
	\caption{The performance of different coarsening ratio.}\label{tab:ratio}
	\centering
\begin{tabular}{lcccccc}\toprule
\multicolumn{1}{c}{\multirow{2}*{\textbf{Ratio}}}& \multicolumn{2}{c}{\textbf{Cora}} & \multicolumn{2}{c}{\textbf{Citeseer}}& \multicolumn{2}{c}{\textbf{Pubmed}}\\ 
\cmidrule(r){2-7}
  &Acc &Macro-F1&Acc &Macro-F1&Acc &Macro-F1\\ \midrule 
  $r=0.3$& \textbf{83.1}$\pm$0.8 & \textbf{82.0}$\pm$0.8  &\textbf{72.4}$\pm$0.9  & 67.6$\pm$0.8 &77.1$\pm$2.9  & 76.9$\pm$2.7\\
  $r=0.5$&82.8$\pm$0.8  & 81.8$\pm$0.8 & \textbf{72.4}$\pm$0.9 & \textbf{67.7}$\pm$0.8  &\textbf{77.3}$\pm$2.9  & \textbf{77.1}$\pm$2.7   \\
$r=0.7$&82.5$\pm$0.9  & 81.4$\pm$0.9 &  72.1$\pm$0.7  & 67.0$\pm$0.7&77.1$\pm$2.8  & 76.8$\pm$2.6  \\
  $r=0.9$&82.4$\pm$0.9  & 81.1$\pm$1.0&  72.0$\pm$0.7  & 66.9$\pm$0.7&77.0$\pm$2.8  & 76.9$\pm$2.6 \\
\bottomrule
\end{tabular}
\end{table}

%% file: appendix.tex
\section{Variance analysis on CSBM} \label{appendix:csbm}
We consider a two-block CSBM denoted as $\mathcal{G}(n, p_1, p_2, q, \mu_1, \mu_2, \sigma^2)$. Here, $A \in \mathbb{R}^{n \times n}$ represents the adjacency matrix of the graph, and $X \in \mathbb{R}^{n \times d}$ represents the feature matrix. In this model, for any two nodes in the graph, the intra-class probability is denoted as $p_i$ ($i=1,2$), and the inter-class probability is denoted as $q$. Additionally, each node's initial feature is independently sampled from a Gaussian distribution $\mathcal{N}(\mu_i, \sigma^2)$.

Our objective is to estimate the variance of node embeddings within each class. Formally, we aim to compute:
\begin{equation}
	\mathbb{E} \left[ \left \Vert D^{-1}AX - \mathbb{E}(D^{-1}AX) \right \Vert_F^2 \right].
\end{equation}

\begin{assumption} [Structural Information]
	$p_1, p_2, q=\Omega(\frac{\log n}{n})$ and $p_1>p_2>q$.
\end{assumption}

\begin{lemma} \label{lemma:jl}
	Assume that $d>\frac{C}{\epsilon^2} \log n$, we have $(1-\epsilon) \Vert A-  \mathbb{E} A   \Vert_F^2 \le   \Vert AX-  \mathbb{E} AX \Vert_F^2  \le 	(1+\epsilon)     \Vert A-  \mathbb{E} A  \Vert_F^2$ with probability at least $1-2\exp(-c{\epsilon^2}d)$.
\end{lemma}
\begin{proof}
Let $X\in \mathbb{R}^{n\times d}$ be the projection matrix which maps each vector $A_i\in \mathbb R^n$ to a $d$-dimensional vector $A_i X\in \mathbb R^d$. Then according to the Johnson\text{-}Lindenstrauss lemma \citep{johnson1984extensions}, given some tolerance $\epsilon$, it holds with probability at least $1-2\exp(-c{\epsilon^2}d)$.
\end{proof}
\begin{lemma} \label{lemma:er}
	For a given Erd\H{o}s-Rényi (ER) graph $\mathcal G(n, p)$, there exists a constant $C$ such that $\Vert A- \mathbb{E} A \Vert \lesssim C \sqrt{np }$
	with probability at least $1-n^{-r}$ for any $r>0$.
\end{lemma}
\begin{proof}
	By using corollary 3.12 from \citep{bandeira2016sharp}, we obtain sharper bounds.
\end{proof}
\begin{lemma} [Sharp concentration, Lemma \ref{lemma:er} and Theorem 3 from \citep{wu2022non}]
	There exists a constant $C$ such that for sufficiently large $n$ , with probability at least $1-{O}(n^{-r})$,
	\begin{equation}
		 \left \Vert D^{-1}A - \mathbb{E}(D^{-1}A) \right \Vert_F  \lesssim \frac{C}{\sqrt{np}}.
	\end{equation}
\end{lemma}
\begin{theorem} 
	Given an CSBM $\mathcal G(n, p_1, p_2, q, \mu_1, \mu_2, \sigma^2)$, it holds that for the variance of class I with intra-class probability $p_1$ is smaller than the the variance of class II with intra-class probability $p_2$.
\end{theorem}
\begin{proof}
According to Lemma \ref{lemma:jl}, to measure the variance of node embedding, we just need to consider structure influence $\left \Vert D^{-1}A - \mathbb{E}(D^{-1}A) \right \Vert_F^2 $.
	By expressing $A=\begin{bmatrix}
		A_1 & B\\
		B & A_2
	\end{bmatrix}$, this allows us to focus on ER graph $\mathcal G(\frac{n}{2}, p_i)$ for each class separately. Denote $R_1=D^{-1}A_1-\mathbb E (D^{-1}A_1)$ and $R_2=D^{-1}A_2-\mathbb E (D^{-1}A_2)$.
We have
 \begin{equation}
     \Vert R_1 \Vert_F^2  \le (1+\epsilon_1) \frac{C}{{np_1}} \le (1-\epsilon_2) \frac{C}{{n p_2}} \le \Vert R_2 \Vert_F^2
 \end{equation}
 for appropriate $\epsilon_1, \epsilon_2$.
 It follows that variance of node embedding decreases more rapidly for the denser class. Furthermore, the convergence speed is inversely proportional to the intra-probability $p_i$, which verifies our observation about community bias amplification.
\end{proof}

\section{The proof of Lemma \ref{lem:weightedReg}}\label{appendix:a}
\begin{lemma}
Let $C_1$ and $C_2$ be the two classes of nodes. Suppose each cluster has the same size $s$, then
\begin{align*}
    &\mathbb{E}_{P\sim \mathcal{P}}\sum_{S_i\in S}\sum_{u\in S_i}s\Vert f(u)-f(S_i)\Vert^2 \\
    =&q_1\sum_{u,v\in C_1,u\neq v}\Vert f(u)-f(v)\Vert^2+q_2\sum_{u,v\in C_2,u\neq v}\Vert f(u)-f(v)\Vert^2 + q_{12}\sum_{u\in C_1,v\in C_2}\Vert f(u)-f(v)\Vert^2.
\end{align*}
\end{lemma}

\begin{proof}
Let $I_u$ be the index such that $u\in S_{I_u}$. We have    
\begin{align}
    \mathbb{E}_{P\sim \mathcal{P}}\sum_{S_i\in S}\sum_{u\in S_i}s\Vert f(u)-f(S_i)\Vert^2 =&\mathbb{E}_{P\sim \mathcal{P}}\sum_{S_i\in S}\frac{1}{2}\sum_{u\in S_i}\sum_{v\in S_i}\Vert f(u)-f(v)\Vert^2 \label{eqn:pointwise-diff}
\end{align}
For fixed $S_i\in S$, without loss of generality we assume $f(S_i)=0$ (if not, redefine $f(u)=f(u)-f(S_i)$), then we have
\begin{align}
\frac{1}{2}\sum_{u\in S_i}\sum_{v\in S_i}\Vert f(u)-f(v)\Vert^2&=\frac{1}{2}\sum_{u\in S_i}\sum_{v\in S_i}(\Vert f(u)\Vert^2 + \Vert f(v)\Vert^2-2 f(u)^Tf(v))\\
&=\sum_{u\in S_i}s\Vert f(u)-f(S_i)\Vert^2.\label{eqn:indicator2}
\end{align}

The equation \ref{eqn:indicator2} is due to the assumption that $f(S_i)=0$. Therefore, 

\begin{align}
    \mathbb{E}_{P\sim \mathcal{P}}\sum_{S_i\in S}\sum_{u\in S_i}s\Vert f(u)-f(S_i)\Vert^2
    =&\mathbb{E}_{P\sim \mathcal{P}}\sum_{u\neq v }\mathbb{I}_{[I_u=I_v]}\Vert f(u)-f(v)\Vert^2.\label{eqn:indicator}
\end{align}

We next divide pairs in (\ref{eqn:indicator}) into three categories and get
\begin{align*}
    &\mathbb{E}_{P\sim \mathcal{P}}\sum_{u\neq v }\mathbb{I}_{[I_u=I_v]}\Vert f(u)-f(v)\Vert^2\\
    =& \mathbb{E}_{P\sim \mathcal{P}}\sum_{u,v\in C_1, u\neq v}\mathbb{I}_{[I_u=I_v]}\Vert f(u)-f(v)\Vert^2+\mathbb{E}_{P\sim \mathcal{P}}\sum_{u,v\in C_2, u\neq v}\mathbb{I}_{[I_u=I_v]}\Vert f(u)-f(v)\Vert^2\\
    &+\mathbb{E}_{P\sim \mathcal{P}}\sum_{u\in C_1,v\in C_2}\mathbb{I}_{[I_u=I_v]}\Vert f(u)-f(v)\Vert^2\\
    =& \sum_{u,v\in C_1,u\neq v}q_1\Vert f(u)-f(v)\Vert^2+\sum_{u,v\in C_2,u\neq v}q_2\Vert f(u)-f(v)\Vert^2+\sum_{u\in C_1,v\in C_2}q_{12}\Vert f(u)-f(v)\Vert^2,
\end{align*}
which finishes the proof. 
\end{proof}

\section{Random Graph Coarsening Algorithm}\label{sec:algorithm}
 
Algorithm \ref{alg:GC} is a detailed description of our random graph coarsening algorithm.
\begin{algorithm}[h]
	\caption{Random Graph Coarsening}\label{alg:GC}
	\begin{algorithmic}[1]
		\Require 
		$G=(A, X)$, threshold $\delta$, the coarsening ratio $r$
		\Ensure 
		Output $G'=(A', X')$
		\State Compute the weight set ${\mathcal{I}}$ of all edges
        \State Construct an edge set ${\mathcal{E}}$ of length $r n$ by randomly selecting the edges according to the weight set ${\mathcal{I}}$. 
		\State Initialize the cluster list $\mathcal{T}$
		\For{$i=0$ to $r n$}
		\State Obtain $(u, v)={\mathcal{E}}_i$
		\State Retrieve the clusters $\mathcal{T}_u$ and $\mathcal{T}_v$ from $\mathcal{T}$, where $\mathcal{T}_u$ contains $u$ and $\mathcal{T}_v$ contains $v$
        \If{$|\mathcal{T}_u|+|\mathcal{T}_v|<\delta$ and $\mathcal{T}_u \neq \mathcal{T}_v$}
		\State Merge cluster ${\mathcal{T}_u}$ and cluster ${\mathcal{T}_v}$ to a new cluster
		\EndIf
		\EndFor
		\State Construct the assignment matrix $P$ by $\mathcal{T}$
        \State Compute the coarsened adjacency matrix $A'= P^T A P$
        \State Compute the coarsened feature matrix $X'=\widetilde{D}^{'-1}P^{T}\widetilde{D}X$\\
		\Return $G'=(A', X')$
	\end{algorithmic}
\end{algorithm}

\section{Comparison of Node Embeddings of the Original Graph and the  Coarsened Graph}\label{sec:error_bound}
In this section, we assume the graph is connected. When there are multiple connected components, each component can be analyzed separately, and thus the conclusion holds for general graphs.

The feature matrix of the coarsened graph is computed using the formula in line 13 of Algorithm~\ref{alg:GC}, which is different from prior work. Here we provide a justification on this. We consider a GNN encoder $Z = \sigma(\hat{A}^kXW)$ with $\hat{A} = \widetilde{D}^{-1}\tilde{A}$. 
Assume the corresponding supernode of $u$ in the coarsened graph is $v$. We use $Z_u$ and $Z'_v$ to represent the node embeddings learned from the original graph $G$ and the coarsened graph $G'$ respectively. We show next that using our coarsened feature matrix, the difference between $Z_u$ and $Z'_v$ converges to zero as $k\rightarrow \infty$.

We assume the activation function $\sigma(\cdot)$ and the linear transformation function $W$ to be Lipschitz continuous. These assumptions are commonly used in previous analyses of GNNs \citep{chen2018stochastic, vikas2020limit, cong2020minimal, cong2021sample}. Then, the coarsening error can be expressed as:
\begin{equation}\label{approximate}
	\|Z_u-Z'_v\|  = \| \sigma (\hat{A}^kXW)_u - \sigma (\hat{A'}^kX'W)_v \| \leq \kappa \|  (\hat{A}^kX)_u  -  (\hat{A'}^kX')_v \|,
\end{equation}
where $\kappa$ represents the Lipschitz constant. For notational convenience, let $\pi^{(k)}_u = (\hat{A}^kX)_u $ and $\pi'^{(k)}_v =  (\hat{A'}^kX')_v$. We need the following Lemma, the proof of which can be found in \cite{chung1997spectral}.
\begin{lemma} \label{lemma_b}
	\begin{equation}
		\hat{A}^{\infty}_{i,j}=\frac{\widetilde{d_j}}{\sum_{u\in G}\widetilde{d_u}}=\frac{\widetilde{d_j}}{2m+n}, \qquad 	|\hat{A}^{k}_{i,j}-\hat{A}^{\infty}_{i,j}| \leq \lambda^{k}_{2} \widetilde{d_i}^{-\frac{1}{2}}\widetilde{d_j}^{\frac{1}{2}},
	\end{equation}
	where $\lambda_{2}$ is the second largest eigenvalue of $\hat{A}$ and $\widetilde{d_i}$ denotes the degree of node $i$ with self-loop.
\end{lemma}

\begin{theorem}\label{theorem_error}
	Let the coarsened feature $X'=\widetilde{D}^{'-1}P^{T}\widetilde{D}X$, then for any node $u$, we have
	\begin{equation}
		\|\pi^{(k)}_u-{\pi'_v}^{(k)}\|\leq  \sqrt{\frac{d_{max}}{d_{min}}} (n\lambda^{k}_{2} +n'\lambda^{'k}_{2}),
	\end{equation}
	where $ \lambda_{2}$ and $ \lambda'_{2}$ are the second largest eigenvalues in $G$ and $G'$ respectively. $d_{max}$ and $d_{min}$ are the maximum degree and minimum degree in $G$ and $G'$, respectively.
\end{theorem}
\begin{proof}
Given the coarsened adjacency matrix $\widetilde{A}'=P^T\widetilde{A}P$, the sum of the weighted edges in $\widetilde{A}'$ is still $2m+n$. If node $j$ in $G'$ is not a supernode, we have $\hat{A'}^{\infty}_{i,j}=\frac{\widetilde{d'_j}}{\sum_{v\in G'}\widetilde{d'_v}}=\frac{\widetilde{d_j}}{2m+n}=\hat{A}^{\infty}_{i,j}$. Let $S$ be the set of supernodes in $G'$, i.e., those nodes in $G'$ containing at least two nodes from the original graph, and define $Q$ as the set of nodes participating in the coarsening process: $Q = \bigcup\limits_{S_i \in S} S_i$. Then we have:
	\begin{equation}
		\small
		\begin{aligned}
			&\|\pi^{(k)}_u-{\pi'_v}^{(k)}\| \\
			=&\| \sum_{S_i\in S} (\hat{A'}^{k}_{v,S_i} \cdot X'_{S_i}-  \sum_{j \in S_i}\hat{A}^{k}_{u,j} \cdot X_{j}) + \sum_{j\in{V\backslash Q} }( \hat{A}^{k}_{u,j} \cdot X_j- \hat{A'}^{k}_{v,j} \cdot X'_j ) \|\\			
			\leq&\| \sum_{S_i\in S} (\hat{A'}^{k}_{v,S_i} \cdot X'_{S_i}-  \sum_{j \in S_i}\hat{A}^{k}_{u,j} \cdot X_{j}) )\| + \| \sum_{j\in{V\backslash Q} }( \hat{A}^{k}_{u,j} \cdot X_j- \hat{A'}^{k}_{v,j} \cdot X'_j ) \|.	\end{aligned}
	\end{equation} 

First, 
	\begin{equation}
		\small
		\begin{aligned}
		&\| \sum_{S_i\in S} (\hat{A'}^{k}_{v,S_i} \cdot X'_{S_i}-  \sum_{j \in S_i}\hat{A}^{k}_{u,j} \cdot X_{j})\| \\
= &\| \sum_{S_i\in S} (\hat{A'}^{k}_{v,S_i} \cdot X'_{S_i} - \hat{A'}^{\infty}_{v,S_i} \cdot X'_{S_i} + \hat{A'}^{\infty}_{v,S_i} \cdot X'_{S_i} 
- \sum_{j \in S_i}(\hat{A}^{k}_{u,j} \cdot X_{j} -\hat{A}^{\infty}_{u,j} \cdot X_{j}+ \hat{A}^{\infty}_{u,j} \cdot X_{j}))\|\\
= &\| \sum_{S_i\in S} (\hat{A'}^{k}_{v,S_i} \cdot X'_{S_i} - \hat{A'}^{\infty}_{v,S_i} \cdot X'_{S_i}
- \sum_{j \in S_i}(\hat{A}^{k}_{u,j} \cdot X_{j} -\hat{A}^{\infty}_{u,j} \cdot X_{j})
+ \hat{A'}^{\infty}_{v,S_i} \cdot X'_{S_i} -
\sum_{j \in S_i}\hat{A}^{\infty}_{u,j} \cdot X_{j})\|\\
\leq &\| \sum_{S_i\in S} (\hat{A'}^{k}_{v,S_i} \cdot X'_{S_i} - \hat{A'}^{\infty}_{v,S_i} \cdot X'_{S_i} )\| + \|
\sum_{S_i\in S} \sum_{j \in S_i}(\hat{A}^{k}_{u,j} \cdot X_{j} -\hat{A}^{\infty}_{u,j} \cdot X_{j})\| 
\\
+ & \|\sum_{S_i\in S}(\hat{A'}^{\infty}_{v,S_i} \cdot X'_{S_i} -
\sum_{j \in S_i}\hat{A}^{\infty}_{u,j} \cdot X_{j})\|\\
\leq &\sum_{S_i\in S} \| \hat{A'}^{k}_{v,S_i} - \hat{A'}^{\infty}_{v,S_i}\|\|X'_{S_i}\| + \sum_{S_i\in S} \sum_{j \in S_i} \|
\hat{A}^{k}_{u,j} -\hat{A}^{\infty}_{u,j}\|\|X_{j}\|
+ \|\sum_{S_i\in S}(\hat{A'}^{\infty}_{v,S_i} \cdot X'_{S_i} -
\sum_{j \in S_i}\hat{A}^{\infty}_{u,j} \cdot X_{j})\|.
		\end{aligned}
	\end{equation}

 For the sake of simplicity, we assume the feature $X_j$ is non-negative and normalize the $X_j$ so that $\|X_j\|=1$.
Let the supernode feature $X'_{S_i}=\frac{\sum_{j \in S_i}\hat{A}^{\infty}_{u,j}X_{j}}{\hat{A'}^{\infty}_{v,S_i}}=\frac{ \sum_{j \in S_i}\widetilde{d}_{j}X_{j}}{\sum_{j \in S_i}\widetilde{d}_{j}}$. In other words, the coarsened feature matrix $X'$ is defined as $\widetilde{D}^{'-1}P^{T}\widetilde{D}X$, which implies $\|X'_j\|\leq 1$. Then, we have
\begin{equation}
\small
\begin{aligned}
&\| \sum_{S_i\in S} (\hat{A'}^{k}_{v,S_i} \cdot X'_{S_i}-  \sum_{j \in S_i}\hat{A}^{k}_{u,j} \cdot X_{j})\| \\
 \leq &\sum_{S_i\in S} \| \hat{A'}^{k}_{v,S_i} - \hat{A'}^{\infty}_{v,S_i}\|\|X'_{S_i}\| + \sum_{S_i\in S} \sum_{j \in S_i} \|
\hat{A}^{k}_{u,j} -\hat{A}^{\infty}_{u,j}\|\|X_{j}\|
+ \|\sum_{S_i\in S}(\hat{A'}^{\infty}_{v,S_i} \cdot X'_{S_i} -
\sum_{j \in S_i}\hat{A}^{\infty}_{u,j} \cdot X_{j})\|\\
\leq & \sum_{S_i\in S} \| \hat{A'}^{k}_{v,S_i} - \hat{A'}^{\infty}_{v,S_i}\| + \sum_{S_i\in S} \sum_{j \in S_i} \|
\hat{A}^{k}_{u,j} -\hat{A}^{\infty}_{u,j}\|\\
\leq& \sum_{S_i\in S} \lambda^{'k}_{2}\widetilde{d'_u}^{-\frac{1}{2}}\widetilde{d'_{S_i}}^{\frac{1}{2}} + \sum_{S_i\in S} \sum_{j \in S_i}\lambda^{k}_{2}\widetilde{d_u}^{-\frac{1}{2}}\widetilde{d_j}^{\frac{1}{2}}. ~~~~~ \text{(by Lemma \ref{lemma_b})}
		\end{aligned}
	\end{equation}
 
	For each uncoarsened node $u\in{V\backslash Q}$, we have $	\hat{A}^{\infty}_{u,j} = \hat{A'}^{\infty}_{v,j} $ and $X_j=X'_j$. Therefore,
	
	\begin{equation}
		\small
		\begin{aligned}
				&\| \sum_{j\in{V\backslash Q} }( \hat{A}^{k}_{u,j} \cdot X_j- \hat{A'}^{k}_{v,j} \cdot X'_j ) \|\\
			=&\| \sum_{j\in{V\backslash Q} }( \hat{A}^{k}_{u,j} \cdot X_j - \hat{A}^{\infty}_{u,j} \cdot X_j + \hat{A}^{\infty}_{u,j} \cdot X_j - \hat{A'}^{k}_{v,j} \cdot X'_j ) \|\\
			\leq  & \sum_{j\in{V\backslash Q} }(\| \hat{A}^{k}_{u,j} \cdot X_j - \hat{A}^{\infty}_{u,j} \cdot X_j \|+ \|\hat{A}^{\infty}_{u,j} \cdot X_j - \hat{A'}^{k}_{v,j} \cdot X'_j  \|)\\
			\leq & \sum_{j\in{V\backslash Q} }(\| \hat{A}^{k}_{u,j}- \hat{A}^{\infty}_{u,j} \|+ \|\hat{A'}^{\infty}_{v,j} - \hat{A'}^{k}_{v,j}\|)\\
			\leq & \sum_{j\in{V\backslash Q}}(\lambda^{k}_{2} \widetilde{d_u}^{-\frac{1}{2}}\widetilde{d_j}^{\frac{1}{2}} +\lambda^{'k}_{2} \widetilde{d'_{v}}^{-\frac{1}{2}}\widetilde{d'_j}^{\frac{1}{2}}). ~~~~~ \text{(by Lemma \ref{lemma_b})}
		\end{aligned}
	\end{equation}
	We define $d_{max}$ and $d_{min}$ as the maximum degree and minimum degree in $G$ and $G'$, respectively. Thus, we have the following conclusion
	\begin{equation}
		\small
		\begin{aligned}
		&\|\pi^{(k)}_u-{\pi^\prime_v}^{(k)}\| \\ 
			\leq& \sqrt{\frac{d_{max}}{d_{min}}}|S| \lambda^{'k}_{2} + \sqrt{\frac{d_{max}}{d_{min}}}|Q| \lambda^{k}_{2} + \sqrt{\frac{d_{max}}{d_{min}}}(n-|Q|)(\lambda^{k}_{2} +\lambda^{'k}_{2})\\
   			=&\sqrt{\frac{d_{max}}{d_{min}}} (n\lambda^{k}_{2} +(n+|S|-|Q|)\lambda^{'k}_{2})\\
      =&\sqrt{\frac{d_{max}}{d_{min}}} (n\lambda^{k}_{2} +n'\lambda^{'k}_{2}).
		\end{aligned}
	\end{equation} 
\end{proof}
According to Theorem \ref{theorem_error} and  inequality (\ref{approximate}), we have the following upper bound:
\begin{equation}
	\|Z_u-Z'_v\|\leq  \kappa \sqrt{\frac{d_{max}}{d_{min}}} (n\lambda^{k}_{2} +n'\lambda^{'k}_{2}).
\end{equation}

On the account of $0 < \lambda_{2} < 1$ and $0  < \lambda'_{2} < 1$, when $k \rightarrow \infty$, the error $\|Z_u-Z'_v\|\rightarrow 0$.

\section{Generalizability of RGCCL}\label{appendix:d}

Following \citep{huang2021towards, wang2022grade}, we investigate the generalizability of our self-supervised model. Denote $\mathcal{G}_{v_i}$ as the ego network of node $v_i$ and the corresponding adjacency matrix as $\hat{A}_{\mathcal{G}_{v_i}}$, \cite{wang2022grade} characterizes the concentration of a graph augmentation set with the following definition.

\begin{definition}[($\alpha, \gamma, \hat{d}$)-Augmentation] The data augmentation set $A$, which includes the original graph, is a ($\alpha, \gamma, \hat{d}$)-augmentation, if for each class $C_k$, there exists a main part $C_k^0\subseteq C_k$ (i.e., $\mathbb{P}[v\in C_k^0]\geq \sigma \mathbb{P}[v\in C_k]$), where $\mathrm{sup}_{v_1, v_2\in C_k^0}d_A(v_1, v_2)\leq \gamma(\frac{\mathcal{D}}{\hat{d}_{\mathrm{min}}^k})^{1/2}$ hold with $d_A(v_i, v_j)=\mathrm{min}_{\mathcal{G}'_i\in A(\mathcal{G}_{v_i}), \mathcal{G}'_j\in A(\mathcal{G}_{v_j})}\Vert (\frac{\hat{A}_{\mathcal{G}'_{i}}}{\hat{d}_{\mathcal{G}'_{i}}} - \frac{\hat{A}_{\mathcal{G}'_{j}}}{\hat{d}_{\mathcal{G}'_{j}}})X\Vert$, and $\hat{d}_{\mathrm{min}}^k$ is the minimum degree in class $C_k$.

\end{definition}

Given a ($\alpha, \gamma, \hat{d}$)-Augmentation, one can establish inequalities between contrastive losses and the dot product of class center pairs. This means classes will be more separable while the objectives are being optimized. For example, \cite{huang2021towards} gives the proofs for InfoNCE and the cross-correlation loss. Since our objective differs from both of them, we prove a similar theorem as \cite{huang2021towards}. Denote $\mu_k :=\mathbb{E}_{v\in C_k}\mathbb{E}_{\mathcal{G}'\in A(\mathcal{G}_v)}[f(\mathcal{G}')]$, $p_k := \mathbb{P}[v\in C_k]$, $S_\epsilon := \{ v\in \bigcup_{k=1}^KC_k:
\forall \mathcal{G}_1, \mathcal{G}_2\in A(\mathcal{G}_v), \Vert f(\mathcal{G}_1)-f(\mathcal{G}_2)\Vert\leq \epsilon\}$ and $R_\epsilon :=1-\mathbb{P}[S_\epsilon]$.

\begin{theorem}
    Assume that encoder $f$ with norm $1$ is $M$-Lipschitz continuous. For a given ($\alpha, \gamma, \hat{d}$) augmentation set $A$, any $\epsilon>0$ and $k\neq l$,
    \begin{equation}
        \mu_k^T\mu_l \leq \frac{1}{p_kp_l}(- \frac{1}{2n^2\mathcal{L}_{neg}} + \tau(\epsilon, \alpha, \gamma, \hat{d})),
    \end{equation}
where $\tau(\epsilon, \alpha, \gamma, \hat{d})=2R_{\epsilon}+16(1-\alpha(1-\frac{1}{2}\epsilon-\frac{1}{4}M\mathrm{max}_k(\gamma\sqrt{\frac{\mathcal{D}}{\hat{d}_\mathrm{min}^k}}))+KR_{\epsilon})^2+8(1-\alpha(1-\frac{1}{2}\epsilon-\frac{1}{4}M\mathrm{max}_k(\gamma\sqrt{\frac{\mathcal{D}}{\hat{d}_\mathrm{min}^k}}))+KR_{\epsilon}+\frac{K-1}{K})$.
\end{theorem}

\begin{proof}
Our negative pair loss is
    \begin{align*}
  \mathcal{L}_{neg}  &= \mathbb{E}_{P\sim \mathcal{P}}[\frac{1}{\sum_{i=1}^{n} \sum_{j=1}^{n} [\Vert h_i\Vert^2 + \Vert h_j\Vert^2 -2h_i^Th_j]}]\\
    &=\frac{1}{2}\mathbb{E}_{P\sim \mathcal{P}}[\frac{1}{n^2-\sum_{i=1}^{n} \sum_{j=1}^{n}h_i^Th_j}].
\end{align*}
Since $\sum_{i=1}^{n} \sum_{j=1}^{n}h_i^Th_j\leq n^2$, and $\frac{1}{n^2-x}$ is convex for $x\leq n^2$. Using Jensen's inequality, we have:
\begin{align*}
    2\mathcal{L}_{neg} &\geq \frac{1}{n^2-\mathbb{E}_{P\sim \mathcal{P}}[\sum_{i=1}^{n} \sum_{j=1}^{n}h_i^Th_j]}\\
    &=\frac{1}{n^2(1-\mathbb{E}_{x_i,x_j}\mathbb{E}_{P\sim \mathcal{P}}[h_i^Th_j])}.
\end{align*}
Next, we focus on $\mathbb{E}_{x_i,x_j}\mathbb{E}_{P\sim \mathcal{P}}h_i^Th_j$:
\begin{align*}
    \mathbb{E}_{x_i,x_j}\mathbb{E}_{P\sim \mathcal{P}}[h_i^Th_j] &\geq \sum_{k=1}^K\sum_{l=1}^K \mathbb{E}_{x_i,x_j}\left[\mathbb{I}(x_i\in S_{\epsilon}\cap C_k)\mathbb{I}(x_j\in C_l)\mathbb{E}_{P\sim \mathcal{P}}(h_i^Th_j)\right] - \mathbb{E}_{x_i,x_j}[\mathbb{I}(x_i\in \Bar{S_{\epsilon}})]\\
    &=\sum_{k=1}^K\sum_{l=1}^K \mathbb{E}_{x_i,x_j}\left[\mathbb{I}(x_i\in C_k)\mathbb{I}(x_j\in C_l)(\mu_k^T\mu_l)\right] - R_{\epsilon} + \Delta_1\\
&=\sum_{k=1}^K\sum_{l=1}^K\left[p_kp_l\mu_k^T\mu_l\right] - R_{\epsilon} + \Delta_1\\
    &\geq p_kp_l\mu_k^T\mu_l + \frac{1}{K} - R_{\epsilon} + \Delta_1,
\end{align*}
where $\Delta_1$ is defined as
\begin{align*}
    \Delta_1 &= \sum_{k=1}^K\sum_{l=1}^K \mathbb{E}_{x_i,x_j}[\mathbb{I}(x_i\in S_{\epsilon}\cap C_k)\mathbb{I}(x_j\in C_l)\mathbb{E}_{P\sim \mathcal{P}}(h_i^Th_j)]\\
    &-  \sum_{k=1}^K\sum_{l=1}^K \mathbb{E}_{x_i,x_j}[\mathbb{I}(x_i\in C_k)\mathbb{I}(x_j\in C_l)(\mu_k^T\mu_l)]\\
    &=-\sum_{k=1}^K\sum_{l=1}^K \mathbb{E}_{x_i,x_j}[(\mathbb{I}(x_i\in C_k)-\mathbb{I}(x_i\in S_{\epsilon}\cap C_k))\mathbb{I}(x_j\in C_l)\mathbb{E}_{P\sim \mathcal{P}}(h_i^Th_j)]\\
    &+\sum_{k=1}^K\sum_{l=1}^K \mathbb{E}_{x_i,x_j}[\mathbb{I}(x_i\in C_k)\mathbb{I}(x_j\in C_l)\mathbb{E}_{P\sim \mathcal{P}}[h_i^Th_j-\mu_k^T\mu_l]].
\end{align*}
Then,
\begin{align*}
    |\Delta_1| \leq & R_{\epsilon}+\sum_{k=1}^K\sum_{l=1}^K \mathbb{E}_{x_i,x_j}\left[\mathbb{I}(x_i\in C_k)\mathbb{I}(x_j\in C_l)\mathbb{E}_{P\sim \mathcal{P}}|h_i^Th_j-\mu_k^T\mu_l|\right]\\
    \leq & R_{\epsilon}+16(1-\alpha(1-\frac{1}{2}\epsilon-\frac{1}{4}M\mathrm{max}_k(\gamma\sqrt{\frac{\mathcal{D}}{\hat{d}_\mathrm{min}^k}}))+KR_{\epsilon})^2\\
    &+8(1-\alpha(1-\frac{1}{2}\epsilon-\frac{1}{4}M\mathrm{max}_k(\gamma\sqrt{\frac{\mathcal{D}}{\hat{d}_\mathrm{min}^k}}))+KR_{\epsilon}).
\end{align*}
Thus, if we define
\begin{align*}
    \tau(\epsilon, \alpha, \gamma, \hat{d})=&2R_{\epsilon}+16(1-\alpha(1-\frac{1}{2}\epsilon-\frac{1}{4}M\mathrm{max}_k(\gamma\sqrt{\frac{\mathcal{D}}{\hat{d}_\mathrm{min}^k}}))+KR_{\epsilon})^2\\
    &+8(1-\alpha(1-\frac{1}{2}\epsilon-\frac{1}{4}M\mathrm{max}_k(\gamma\sqrt{\frac{\mathcal{D}}{\hat{d}_\mathrm{min}^k}}))+KR_{\epsilon}+\frac{K-1}{K}),
\end{align*}
we have
\begin{align*}
    \mu_k^T\mu_l &\leq \frac{1}{p_kp_l}(1 - \frac{1}{2n^2\mathcal{L}_{neg}}-\frac{1}{K}+R_{\epsilon}+|\Delta_1|)\\
    &\leq \frac{1}{p_kp_l}(- \frac{1}{2n^2\mathcal{L}_{neg}} + \tau(\epsilon, \alpha, \gamma, \hat{d})).
\end{align*}
This finishes the proof.
\end{proof}

Next we discuss the concentration property of the random coarsening augmentation. First notice that for any two nodes $v_i$ and $v_j$, the probability that they are coarsened together is equal to the probability that they are connected by randomly selected edges in the algorithm. Suppose the edges are selected independent and identically with probability $p$, then the connection probability is lower bounded by $p^{\mathrm{dia}(G)}$, where $\mathrm{dia}(G)$ is the diameter of $G$. Once $u$ and $v$ are coarsened together in at least one coarsened graph, $d(u,v) = 0$, which means our random coarsening augmentation can be very well-concentrated.

\section{Experimental details}\label{sec:details}

For all unsupervised models, the learned representations are evaluated by training and testing a logistic regression classifier except for Ogbn-Arxiv. Since Ogbn-arxiv exhibits more complex characteristics, we use a more powerful MLP classifier. The detailed statistics of the six datasets are summarized in Table \ref{tab:datasets}.

\begin{table*}[!htbp]\small
	\caption{Summary of the datasets used in our experiments}.	
	\begin{tabular}{l|cccc}\toprule \centering
	\textbf{Dataset}&\textbf{Nodes}&\textbf{Features}&\textbf{Classes}&\textbf{Avg. Degree}\\
	\midrule
	Cora&2,708 &1,433&7&3.907\\
	Citeseer&3,327&3,703&6&2.74\\
	Pubmed&19,717&500&3&4.50\\
	Amazon-Photo&7,650 &745&8&31.13\\
	Amazon-Computers&13,752& 767 &10&35.76\\
        Ogbn-Arxiv&169,343   & 128 &  40&13.67 \\
\bottomrule
	\end{tabular}
	\label{tab:datasets}
	\centering
\end{table*}

\noindent\paragraph{Details of our model} In our model, we use SGC as the encoder for Cora, Citeseer, Pubmed, while we use GCN as the encoder for Photo and Computers. 
The detailed hyperparameter settings are listed in Table \ref{tab:Hyper-parameters}.

\begin{table}[!thbp]
\setlength{\abovecaptionskip}{0.05cm}
	\caption{Summary of the  hyper-parameters.}\label{tab:Hyper-parameters}
	\centering
\begin{tabular}{lccccc}\toprule
 \textbf{Dataset} & \textbf{Epoch}& \textbf{Learning rate}& \textbf{$\alpha$}&\textbf{$\beta$}\\
\midrule 
Cora & 25 &  0.01& 15000&500 \\
Citeseer & 200 &0.0002  & 15000& 500\\
Pubmed &  25&0.02  & 20000&200 \\
Amazon-Photo &20  &0.001  & 100000&100000 \\
Amazon-Computers & 20 & 0.0002 &20000 & 20000\\
Ogbn-Arxiv & 10 & 0.0001 &  2000000&  200000\\
\bottomrule
\end{tabular}
\end{table}

\noindent\paragraph{Details of Baselines} 
We compare \ourmodel~ with state-of-the-art GCL models DGI\footnote{\small DGI (MIT License): \url{https://github.com/pyg-team/pytorch_geometric/blob/master/examples/infomax_transductive.py}}, GRACE\footnote{\small GRACE (Apache License 2.0):  \url{https://github.com/CRIPAC-DIG/GRACE}}, GraphCL\footnote{\small GraphCL (MIT License): 
  \url{https://github.com/Shen-Lab/GraphCL}}, GCA\footnote{\small GCA (MIT License): 
  \url{https://github.com/CRIPAC-DIG/GCA}}, 
  CCA-SSG\footnote{\small CCA-SSG (Apache License 2.0): 
  \url{https://github.com/hengruizhang98/CCA-SSG}}, gCooL\footnote{\small gCooL (MIT License):
  \url{https://github.com/lblaoke/gCooL}}, GGD\footnote{\small GGD (MIT License): 
\url{https://github.com/zyzisastudyreallyhardguy/Graph-Group-Discrimination}}, GRADE\footnote{\small GRADE 
 (MIT License): 
  \url{https://github.com/BUPT-GAMMA/Uncovering-the-Structural-Fairness-in-Graph-Contrastive-Learning}} and classic graph embedding model Deepwalk. For all the baseline models, we use the source code from corresponding repositories. Due to the large scale of  Ogbn-Arixv, some GCL models are unable to process the full-graph on GPU because of memory limitations. As a result, we apply graph sampling techniques to train these models.

\noindent\paragraph{Configuration} All the algorithms and models are implemented in Python and PyTorch Geometric. Experiments are conducted on a server with an NVIDIA 3090 GPU (24 GB memory) and an Intel(R) Xeon(R) Silver 4210R CPU @ 2.40GHz.

\noindent\paragraph{Memory usage} Figure \ref{fig:more_memory} shows the memory usage of our model and 6 mainstream GCL models on Citeseer and Pubmed. \ourmodel~exhibits the same level of memory usage as GGD, which is specifically designed to save computation costs; our model benefits from the effective reduction of graph size via random coarsening.

\begin{figure*}[htbp]
	\setlength{\abovecaptionskip}{0.1cm}
	\setlength{\belowcaptionskip}{-0.3cm}
	\centering
 \subfigure[Cora]{\includegraphics[width=0.30\textwidth]{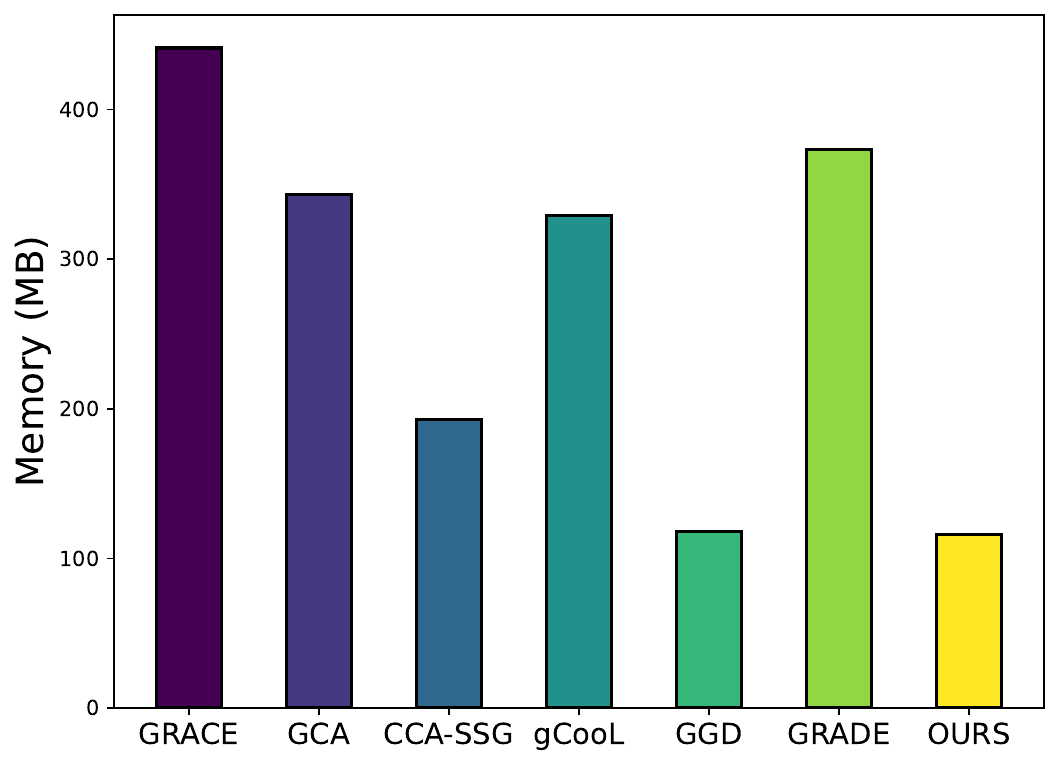}}
	\subfigure[Citeseer]{\includegraphics[width=0.3\textwidth]{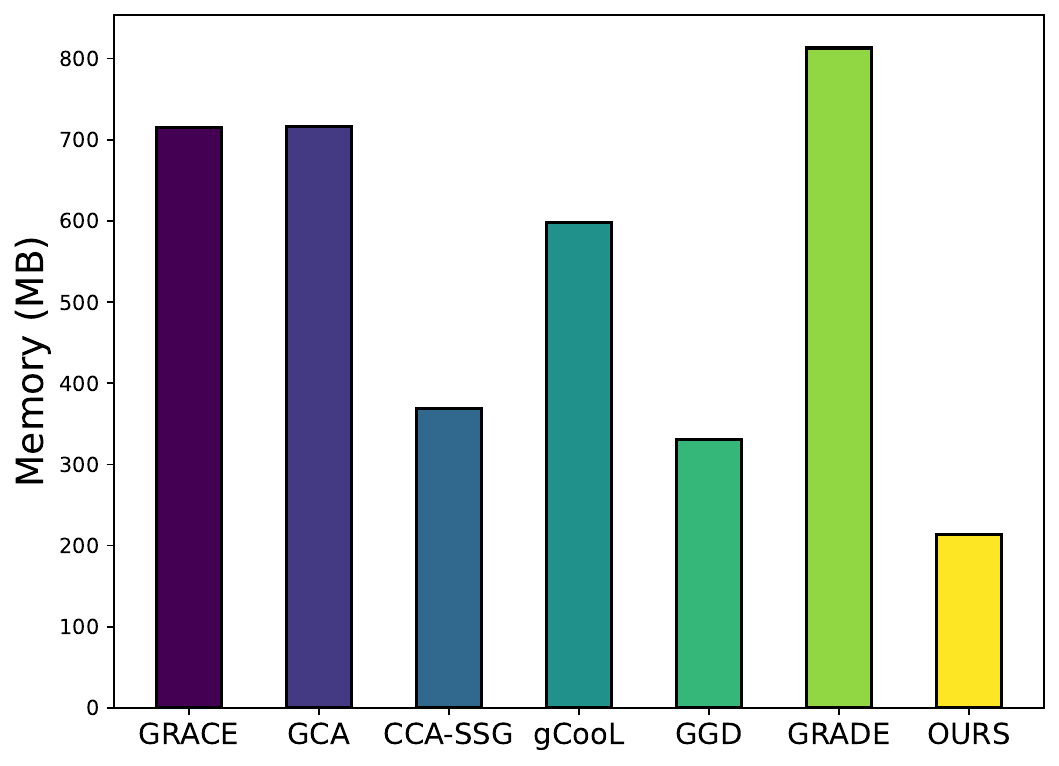}}
	\subfigure[Pubmed]{\includegraphics[width=0.30\textwidth]{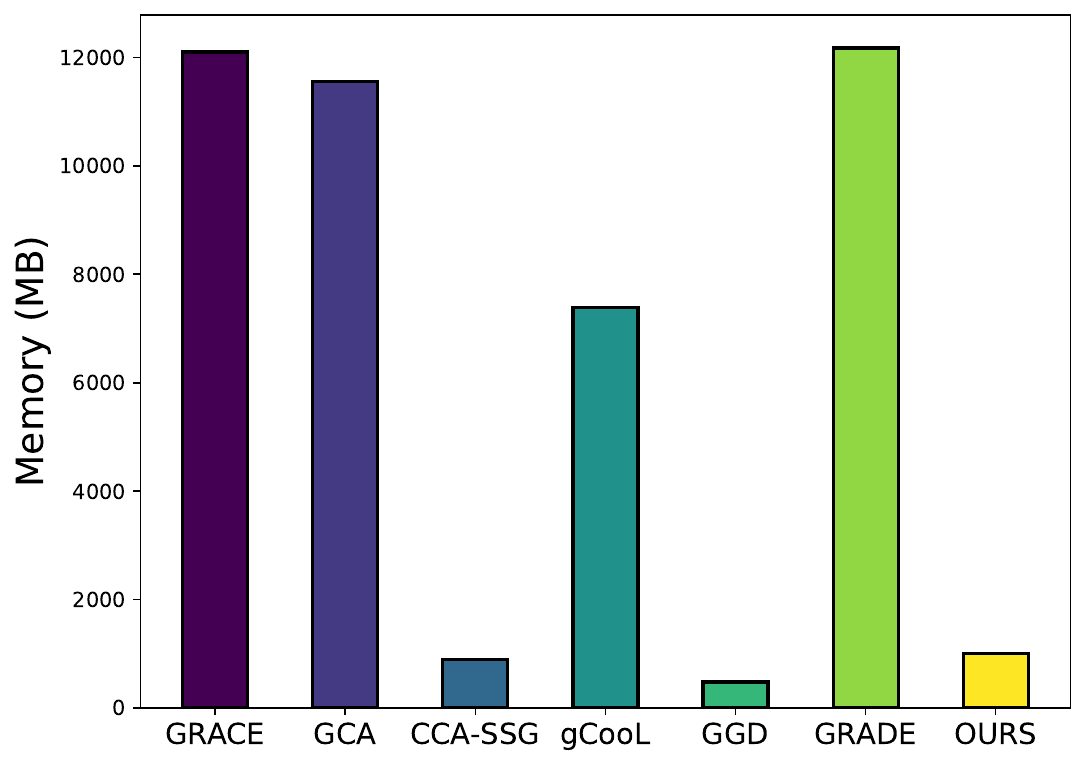}}
	\caption{The memory usage of baselines and RGCCL on Citeseer and Pubmed.}
 	\label{fig:more_memory}
\end{figure*}

\noindent\paragraph{More quantitative analysis about community bias amplification} To further demonstrate that RGCCL effectively alleviates the amplification of community bias, we use the Matthew's coefficient to measure community bias. Table \ref{mcc} presents the results of Matthew's coefficient for representative GCL models and our RGCCL. The results indicate that the bias in the embeddings learned by RGCCL is significantly less than that of other GCL models.

\begin{table*}[!htbp]
\caption{Matthew's coefficient for RGCCL and four baselines.}\label{mcc}
	\centering
    \begin{tabular}{l|ccc}
    \toprule
         & Cora  & CiteSeer & PubMed\\
         \midrule
         DGI&73.0$\pm$1.5&64.5$\pm$1.2&57.8$\pm$4.3\\
         GRACE&67.9$\pm$1.6&60.0$\pm$2.0&62.3$\pm$3.8\\
         CCA-SSG&74.5$\pm$1.6&64.6$\pm$1.4&64.0$\pm$4.1\\
         GGD&77.0$\pm$1.6&64.9$\pm$1.2&63.7$\pm$4.1\\
         GRADE&75.7$\pm$1.5&62.1$\pm$1.3&58.1$\pm$3.4\\
         RGCCL&78.9$\pm$0.9 & 66.3$\pm$0.8&65.6$\pm$4.2\\
         \bottomrule
    \end{tabular}
\end{table*}